\documentclass[journal,10pt,twoside]{IEEEtran}

\usepackage[
  a4paper,
  left=18mm,
  right=18mm,
  top=24mm,
  bottom=22mm,
  columnsep=6mm
]{geometry}

\usepackage[T1]{fontenc}
\usepackage[utf8]{inputenc}

\usepackage{amsmath,amssymb,amsfonts,amsthm,mathtools}

\usepackage{newtxtext}
\usepackage{newtxmath}

\usepackage{bm}

\newtheorem{definition}{Definition}
\newtheorem{proposition}[definition]{Proposition}
\newtheorem{theorem}[definition]{Theorem}
\newtheorem{corollary}[definition]{Corollary}

\theoremstyle{remark}
\newtheorem{remark}[definition]{Remark}

\usepackage{graphicx}
\graphicspath{{figures/}}

\usepackage{booktabs}
\usepackage{array}

\usepackage[font=footnotesize]{caption}
\usepackage[font=footnotesize]{subcaption}

\usepackage{xcolor}
\usepackage{tikz}
\usepackage{pgfplots}

\pgfplotsset{compat=1.18}

\usetikzlibrary{
  arrows.meta,
  calc,
  positioning,
  decorations.pathreplacing
}

\definecolor{fiberpositive}{RGB}{31,119,180}
\definecolor{fibernegative}{RGB}{214,85,44}
\definecolor{fiberzero}{RGB}{55,55,55}
\definecolor{fiberboundary}{RGB}{95,95,95}
\definecolor{fiberhighlight}{RGB}{112,48,160}
\definecolor{fiberregion}{RGB}{236,242,247}

\usepackage{cite}
\usepackage[hidelinks]{hyperref}
\usepackage[nameinlink,noabbrev]{cleveref}

\crefformat{equation}{#2(#1)#3}
\Crefformat{equation}{Equation~#2(#1)#3}

\crefname{figure}{Fig.}{Figs.}
\Crefname{figure}{Figure}{Figures}

\crefname{table}{Table}{Tables}
\Crefname{table}{Table}{Tables}

\crefname{section}{Sec.}{Secs.}
\Crefname{section}{Section}{Sections}

\crefname{subsection}{Sec.}{Secs.}
\Crefname{subsection}{Section}{Sections}

\crefname{definition}{Definition}{Definitions}
\Crefname{definition}{Definition}{Definitions}

\crefname{proposition}{Proposition}{Propositions}
\Crefname{proposition}{Proposition}{Propositions}

\crefname{theorem}{Theorem}{Theorems}
\Crefname{theorem}{Theorem}{Theorems}

\crefname{corollary}{Corollary}{Corollaries}
\Crefname{corollary}{Corollary}{Corollaries}

\crefname{remark}{Remark}{Remarks}
\Crefname{remark}{Remark}{Remarks}

\usepackage{enumitem}

\setlist[enumerate]{
  leftmargin=*,
  itemsep=0.3ex,
  topsep=0.5ex,
  parsep=0pt
}

\usepackage{titlesec}

\renewcommand{\thesection}{\arabic{section}}
\renewcommand{\thesubsection}{\thesection.\arabic{subsection}}
\renewcommand{\thesubsubsection}{%
  \thesubsection.\arabic{subsubsection}%
}

\titleformat{\section}
  {\normalfont\large\bfseries}
  {\thesection}
  {0.65em}
  {}

\titleformat{\subsection}
  {\normalfont\normalsize\bfseries}
  {\thesubsection}
  {0.65em}
  {}

\titleformat{\subsubsection}
  {\normalfont\normalsize\itshape}
  {\thesubsubsection}
  {0.65em}
  {}

\titlespacing*{\section}
  {0pt}
  {2.0ex plus 0.5ex minus 0.2ex}
  {0.8ex}

\titlespacing*{\subsection}
  {0pt}
  {1.5ex plus 0.4ex minus 0.2ex}
  {0.5ex}

\titlespacing*{\subsubsection}
  {0pt}
  {1.2ex plus 0.3ex minus 0.2ex}
  {0.4ex}

\newcommand{\R}{\mathbb{R}}
\newcommand{\dd}{\mathrm{d}}
\newcommand{\T}{\mathsf{T}}

\newcommand{\diag}{\operatorname{diag}}
\newcommand{\rank}{\operatorname{rank}}
\newcommand{\argmax}{\operatorname*{arg\,max}}
\newcommand{\Vol}{\operatorname{Vol}}

\newcommand{\Vspace}{\mathcal V}
\newcommand{\Wspace}{\mathcal W}
\newcommand{\Fibre}{\mathcal F}
\newcommand{\Ereg}{\mathcal E_{\mathrm{reg}}}

\newcommand{\papertitle}{%
Drag-Aware Aerodynamic Manipulability for Torque-Limited Redundant
Multirotors: Aerodynamic Promptness based on the Symmetric
Acceleration Capacity%
}

\newcommand{\paperauthor}{Antonio Franchi}

  \title{\papertitle}
  \author{\paperauthor}
  \hypersetup{
    pdftitle={\papertitle},
    pdfauthor={Antonio Franchi}
  }

\begin{document}

\pagestyle{plain}

\makeatletter
\twocolumn[{
\begin{@twocolumnfalse}

\vspace*{-1.5em}

{\raggedright
\fontsize{17}{20}\selectfont
\bfseries
\papertitle
\par}

\vspace{0.9em}

{\raggedright
\normalsize
\textbf{Antonio Franchi}\textsuperscript{1,2}
\par}

\vspace{0.35em}

{\raggedright
\small
\textsuperscript{1}Robotics and Mechatronics,
Faculty of Electrical Engineering, Mathematics and Computer Science,
University of Twente, Enschede, The Netherlands
\par

\textsuperscript{2}Department of Computer, Control and Management Engineering,
Sapienza University of Rome, Rome, Italy
\par

\textit{Corresponding author:}
\href{mailto:schol@r-franchi.eu}{schol@r-franchi.eu}
\par}

\vspace{1.1em}

\noindent
{\small\bfseries Abstract}
\par
\vspace{0.25em}

{\small
\noindent
Aerodynamic promptness quantifies how rotor-speed variations generate
multirotor wrench variations, but its Euclidean formulation assigns the same
local cost to a given rotor acceleration at every operating speed. This work
develops a capacity-aware extension for redundant multirotors with arbitrary
numbers of heterogeneous rotors and wrench components. Under bounded motor
torque and aerodynamic drag, each generally asymmetric instantaneous
rotor-acceleration interval contains a largest zero-centered subset whose
radius defines the symmetric acceleration capacity (SAC). The SAC induces a
Riemannian metric on the positive-capacity rotor-speed region. Propagating its
co-metric through the nonlinear rotor-speed-to-wrench differential yields a
state-attached task-rate capability matrix and ellipsoid. The corresponding
inverse quadratic form equals the minimum normalized rotor-acceleration effort
required to realize a prescribed wrench rate, while the ellipsoid volume
defines the drag-aware aerodynamic manipulability (DAAM) index. Fiberwise DAAM
maximization provides a task-coordinate-invariant criterion for selecting
task-equivalent actuator states; its maximizing set exists on compact regular
domains and can be nonconvex and set valued. Low-dimensional two- and
three-rotor studies make the resulting fiberwise geometry and parameter
dependence directly visible. A complementary two-rotor use case shows how DAAM
can inform a continuous allocation section subject to directional motor-torque
feasibility. For two heterogeneous propulsion systems, the resulting sections
reduce saturation-induced force-tracking degradation relative to the
pseudoinverse in the faster command bands.
\par}

\vspace{0.8em}

{\small
\par}

\vspace{1.6em}

\end{@twocolumnfalse}
}]
\makeatother

\section{Introduction}
\label{sec:introduction}

Actuation redundancy is increasingly used in multirotor aerial vehicles to enlarge wrench-generation capabilities, improve fault tolerance, and support physical interaction. Tiltable-rotor and fixed-tilt platforms demonstrate how propulsion geometry can provide full or redundant wrench actuation \cite{Ryll2015TCST,Rajappa2015ICRA}. Omnidirectional designs further exploit this freedom to generate forces and moments independently \cite{Brescianini2016ICRA,2018e-TogFra}. Related platforms include tiltable-propeller hexarotors and aerial-manipulation systems with omnidirectional wrench authority \cite{kamel2018voliro,park2018_tmech_odar}. The relation between multirotor architecture and achievable task-space authority has also been studied systematically \cite{Hamandi2021IJRR,Michieletto2018TRO}.

For these vehicles, multiple actuator choices can generate the same commanded wrench. Control allocation resolves this freedom while accounting for actuator limits and secondary objectives. Early formulations include direct allocation and constrained optimization \cite{Enns1998GNC,Durham1993JGCD}. Subsequent studies compared numerical allocation methods and organized the broader field within a common framework \cite{Bodson2002JGCD,johansen2013control}. In redundantly actuated multirotors, null-space and optimization-based methods have been used to manage actuator constraints, aerodynamic interactions, and power consumption \cite{NullspaceUAV2021,SuIROS2022}. Energy-optimal allocation has also been demonstrated for an omnidirectional aerial vehicle~\cite{dyer2019energy}.

Allocation quality is not determined only by the current wrench or steady-state energy consumption. The selected rotor-speed state also determines how effectively subsequent rotor-speed variations can change the wrench, which is relevant during demanding trajectories, disturbances, and actuator saturation. Time-optimal planning and nonlinear control have enabled increasingly dynamic quadrotor trajectories \cite{foehn2021_scirobotics_timeopt,Sun2022TROComparative}. Adaptive predictive control has addressed accurate tracking under payload and model variations \cite{Hanover2021RAL_L1NMPC}. Robust tracking under wind and parameter uncertainty has been studied through intelligent active-force control \cite{AbdelmaksoudMailah2024IJDC}, while learning-based methods have addressed nonlinear quadrotor control \cite{Norouzi2026IJDC}. Disturbance-rejection methods have also been developed for precision UAV altitude control \cite{Choe2026IJDC}. Distinct actuator states generating the same wrench can therefore provide different local capabilities for producing subsequent wrench variations.

The local wrench-rate capability of an allocation depends on both the nonlinear propulsion map and the rotor dynamics. Propeller aerodynamics determine how rotor-speed variations affect thrust \cite{Bristeau2009ECC}, while enhanced actuator models improve the prediction of multirotor dynamics and constraints \cite{Bicego2020JIRS}. Static wrench sets and wrench polytopes characterize whether a task wrench is feasible under given actuator bounds \cite{Orsolino2017AWP,FeasibleWrenchSet2022}. The present question is complementary: how does this local capability vary among task-equivalent rotor-speed states?

Manipulability provides a geometric language for local task capability. Classical manipulability maps a bounded set of joint velocities to an ellipsoid of task velocities \cite{Yoshikawa1985IJRR}, while dynamic manipulability incorporates actuator effort and system dynamics into an acceleration ellipsoid \cite{YoshikawaDyn1985}. Differential-geometric formulations have clarified the intrinsic and coordinate-dependent aspects of such measures \cite{ParkBrockett1994IJRR,Staffetti2002Invariance}. Manipulability has also been extended to mechanisms with more general actuation structures \cite{ParkKim1998JMD}. These studies show that a task differential becomes a physical capability measure only after specifying the admissible actuator-side resource.

Previous work applied this principle to multirotor propulsion through \emph{aerodynamic promptness} \cite{Franchi2026ICUAS}. That formulation propagated Euclidean-bounded rotor-speed variations through the nonlinear rotor-speed-to-wrench map and showed that task-preserving internal states can trade energy consumption against active wrench responsiveness. In antagonistic propulsion, aerodynamic co-contraction can increase promptness without changing the commanded wrench. Complementary work showed that the same internal actuation mechanism can modify passive incremental aerodynamic damping at a trim \cite{Franchi2026RAL}.

The Euclidean promptness model nevertheless assigns the same normalized cost to a given rotor acceleration at every operating speed. This is inconsistent with a torque-limited rotor subject to aerodynamic drag. At low rotor speed, the signed-quadratic thrust map has a small differential. At high rotor speed, drag consumes an increasing fraction of the available motor torque and reduces the acceleration authority remaining in both directions. Local wrench-rate capability is therefore limited by different mechanisms at the two ends of the usable speed range.

To address this limitation, this work develops a capacity-aware extension of aerodynamic promptness for general redundant multirotor actuation. The formulation allows an arbitrary number of rotors with different torque limits, drag coefficients, inertias, and fixed contributions to the generated wrench. Starting from the bounded-torque rotor dynamics, we extract the largest zero-centered subset of each generally asymmetric rotor-acceleration interval. Its radius defines the \emph{symmetric acceleration capacity} (SAC), namely the acceleration magnitude guaranteed in either direction. These state-dependent capacities induce an actuator metric, whose co-metric is propagated through the nonlinear actuation differential to obtain the \emph{drag-aware aerodynamic manipulability} (DAAM) objective and its state-attached task-rate capability ellipsoid.

The resulting geometry supports both local wrench-rate generation and redundancy resolution. At a fixed actuator state, the capacity-horizontal lift realizes a prescribed wrench rate with minimum normalized actuator effort. Along an allocation fiber, the DAAM volume instead compares finite task-equivalent actuator states according to their local bidirectional wrench-rate capabilities. We establish existence of fiberwise maximizers on compact regular domains, prove that the maximizing set is invariant under smooth changes of task coordinates, and derive first-order conditions for interior candidates. The analysis also exposes why the selection problem can be nonconvex and set valued.

The theory is illustrated and assessed through two complementary low-dimensional studies. First, two- and three-rotor examples make the actuator metric, allocation fibers, DAAM landscapes, and fiberwise maximizing sets directly visible. These cases reveal how heterogeneous propulsion parameters reshape the capability landscape and how multiple branches, sign-chamber changes, and boundary encounters can arise. Their low dimension is used deliberately to expose geometric features that would be difficult to visualize in a general multirotor actuator space.

Second, a two-rotor scalar-task case demonstrates how the DAAM field and the SAC metric can inform control allocation. A continuous section is constructed by favoring high DAAM while enforcing directional motor-torque feasibility over a prescribed force-rate envelope. Simulations with two heterogeneous propulsion systems show how this capability-aware allocation can reduce saturation-induced force-tracking degradation relative to the pseudoinverse in the faster command bands. This study provides one transparent use case of the general theory rather than a general allocator for arbitrary multirotor systems.

The principal contributions of this work are therefore:
\begin{enumerate}
\item the SAC for heterogeneous torque-limited multirotors and the actuator-capacity geometry that it induces, derived from the bidirectional acceleration authority remaining under aerodynamic drag;

\item the associated task-rate capability matrix and ellipsoid, with an exact minimum-normalized-effort interpretation, and the resulting DAAM index as their volume measure;

\item a fiberwise operating-point selection theory based on DAAM maximization, comprising existence, task-coordinate invariance, interior optimality conditions, and the characterization of nonconvex and potentially set-valued maximizing sets;

\item numerical studies that make the fiberwise geometry and its parameter dependence visible and demonstrate, in a directly interpretable low-dimensional case, how DAAM can inform the construction of a continuous allocation section.
\end{enumerate}

The theoretical construction in the first three contributions applies to arbitrary numbers of rotors and wrench components, with heterogeneous torque limits, drag coefficients, inertias, and fixed contributions to the generated wrench. This work thereby establishes a general foundation for developing capability-aware allocators tailored to particular multirotor architectures, task dimensions, operating domains, and actuator constraints.

The scope remains model based and foundational for subsequent control-allocation developments. The centered metric deliberately omits the additional acceleration available in the easier direction, while the continuous-section demonstration concerns a scalar task with two cooperative rotors. The simulations therefore provide an allocation-level use case under the considered rotor model, rather than a general allocation algorithm or evidence of improved vehicle-level closed-loop performance.

\section{Prior Foundation and Problem Formulation}
\label{sec:foundation}

Consider a multirotor aerial vehicle equipped with \(n\) fixed propellers and redundantly actuated in an \(m\)-dimensional body-wrench space, with \(m<n\). The inequality \(m<n\) provides actuation redundancy: multiple rotor-speed states can generate the same body wrench. The propellers may have different positions, orientations, rotation directions, and aerodynamic coefficients. Their fixed arrangement determines how the individual thrust forces and drag moments combine into the net body wrench, as formalized by general multirotor actuation models and classifications \cite{Michieletto2018TRO,Hamandi2021IJRR}.

Let \(v=(v_1,\ldots,v_n)\in\Vspace\subset\R^n\) collect the rotor-speed coordinates. For a fixed-pitch reversible propeller, the signed thrust magnitude and the associated aerodynamic drag moment are proportional to \(v_i|v_i|\). Because the propeller position, axis, and rotation direction are fixed in the body frame, the complete wrench contribution of rotor \(i\) is therefore \(a_i v_i|v_i|\), where \(a_i\in\R^m\) is its fixed wrench coefficient. Collecting these coefficients as the columns of \(A\in\R^{m\times n}\) yields the signed-quadratic actuation map
\begin{equation}
 f:\Vspace\to\Wspace,
 \qquad
 w=f(v)=A\bigl(v\odot |v|\bigr),
 \label{eq:task-map}
\end{equation}
where \(w\in\Wspace\subseteq\R^m\) collects the selected body-wrench coordinates and \(\odot\) denotes entrywise multiplication. The subsequent analysis depends only on the map \(f\) and on the individual rotor dynamics; the particular propeller arrangement is entirely represented by the columns \(a_i\) of \(A\).

Since the scalar function \(s\mapsto s|s|\) is continuously differentiable also at \(s=0\), the map in \eqref{eq:task-map} is continuously differentiable, with Jacobian
\begin{equation}
 J_f(v)=2A\diag(|v_1|,\ldots,|v_n|).
 \label{eq:task-jacobian}
\end{equation}
Define the regular region
\begin{equation}
 \Vspace_{\mathrm{reg}}
 =\{v\in\Vspace:\rank J_f(v)=m\}.
 \label{eq:regular-region}
\end{equation}
For a regular task value $w$, the corresponding allocation fiber is
\begin{equation}
 \Fibre_w=f^{-1}(w)\cap\Vspace_{\mathrm{reg}}.
 \label{eq:fiber}
\end{equation}
Its tangent space at $v\in\Fibre_w$ is $T_v\Fibre_w=\ker J_f(v)$.

Under a Euclidean rotor-rate metric, the local task-rate matrix is
\begin{equation}
 M_E(v)=J_f(v)J_f(v)^\T.
 \label{eq:euclidean-promptness}
\end{equation}
The associated volume proxy, $\sqrt{\det M_E(v)}$, is the Euclidean aerodynamic promptness introduced in \cite{Franchi2026ICUAS}. Equation \eqref{eq:euclidean-promptness} corresponds to treating the Euclidean unit ball in rotor-rate space as the admissible local resource. The problem considered here is to replace that uniform resource model with the bidirectional acceleration authority guaranteed in the less favorable direction at the actual rotor-speed state.

\section{Drag-Limited Symmetric Acceleration Geometry}
\label{sec:sac}

\subsection{Directional Rotor-Acceleration Interval}

Consider the first-order dynamics of rotor $i$,
\begin{equation}
 m_i\dot v_i=-b_iv_i|v_i|+\tau_i,
 \qquad
 \tau_i\in[-\bar\tau_i,\bar\tau_i],
 \label{eq:rotor-dynamics}
\end{equation}
where $m_i>0$ is the lumped rotational inertia, $b_i>0$ is the aerodynamic drag coefficient, and $\bar\tau_i>0$ is the torque bound. At a fixed $v_i$, the complete instantaneous acceleration interval is
\begin{equation}
 \mathcal A_i(v_i)
 =
 \left[
 \frac{-\bar\tau_i-b_iv_i|v_i|}{m_i},
 \frac{\bar\tau_i-b_iv_i|v_i|}{m_i}
 \right].
 \label{eq:directional-interval}
\end{equation}
This interval is generally not centered at zero. For $v_i>0$, negative acceleration benefits from both negative motor torque and aerodynamic drag, whereas positive acceleration must overcome drag. The roles exchange for $v_i<0$.

\begin{definition}[Symmetric acceleration capacity]
The symmetric acceleration capacity (SAC) at \(v_i\) is the largest \(r\geq 0\) such that
\[
    [-r,r]\subseteq\mathcal A_i(v_i).
\]
\end{definition}

Figure~\ref{fig:single-rotor-sac} illustrates the resulting symmetric
acceleration capacity for representative torque--drag rotor parameters.
In each case, the capacity is maximal at zero rotor speed and decreases
with \(\lvert v_i\rvert\) until it vanishes at the corresponding
positive-capacity boundary.

\begin{figure}[t]
    \centering
    \begin{tikzpicture}
        \begin{axis}[
            width=\linewidth,
            height=6cm,
            xlabel={Rotor speed \(v_i\)},
            ylabel={SAC \(\alpha_i(v_i)\)},
            domain=-12:12,
            samples=100,
            axis lines=center,
            grid=major,
            legend style={
                at={(0.15,1)},
                anchor=north east,
                nodes={align=left},
                font=\footnotesize
            },
            ymin=0,
            ymax=12,
            xmin=-13,
            xmax=13
        ]
            \addplot[blue, thick]
                {max(0,(10-0.1*x^2)/1)};
            \addlegendentry{
                \(m_i=1\),\\
                \(\bar{\tau}_i=10\),\\
                \(b_i=0.1\)
            }

            \addplot[red, thick, dashed]
                {max(0,(10-0.2*x^2)/1)};
            \addlegendentry{
                \(m_i=1\),\\
                \(\bar{\tau}_i=10\),\\
                \(b_i=0.2\)
            }

            \addplot[green!70!black, thick, dotted]
                {max(0,(10-0.1*x^2)/2)};
            \addlegendentry{
                \(m_i=2\),\\
                \(\bar{\tau}_i=10\),\\
                \(b_i=0.1\)
            }
        \end{axis}
    \end{tikzpicture}
    \caption{Symmetric acceleration capacity \(\alpha_i(v_i)\) for
    representative torque--drag rotor parameters. Each curve reaches
    \(\bar{\tau}_i/m_i\) at zero rotor speed and vanishes at
    \(\lvert v_i\rvert=\sqrt{\bar{\tau}_i/b_i}\). The complete directional
    acceleration interval is generally asymmetric, whereas
    \(\alpha_i(v_i)\) is the radius of its largest zero-centered subset.}
    \label{fig:single-rotor-sac}
\end{figure}

\begin{proposition}
For the rotor model \eqref{eq:rotor-dynamics}, the symmetric acceleration capacity is
\begin{equation}
 \alpha_i(v_i)
 =
 \frac{\bar\tau_i-b_iv_i^2}{m_i}
 \label{eq:sac}
\end{equation}
whenever $|v_i|<v_{i,\max}$, where
\begin{equation}
 v_{i,\max}=\sqrt{\frac{\bar\tau_i}{b_i}}.
 \label{eq:vmax}
\end{equation}
\end{proposition}

\begin{proof}
Write the endpoints of \eqref{eq:directional-interval} as $\underline\alpha_i(v_i)$ and $\overline\alpha_i(v_i)$. The radius of the largest centered subset is $\min\{\overline\alpha_i(v_i),-\underline\alpha_i(v_i)\}$. If $v_i\geq0$, these two nonnegative directional magnitudes are $(\bar\tau_i-b_iv_i^2)/m_i$ and $(\bar\tau_i+b_iv_i^2)/m_i$, respectively. If $v_i\leq0$, their order is exchanged. Their minimum is therefore \eqref{eq:sac}. It is positive exactly on the interval in \eqref{eq:vmax}.
\end{proof}

Thus, $\alpha_i$ does not approximate the full interval \eqref{eq:directional-interval}. It exactly records the part guaranteed to remain available in both directions. The additional authority in the easier direction is deliberately discarded.

Define the positive-capacity region
\begin{equation}
 \mathcal E
 =
 \prod_{i=1}^n(-v_{i,\max},v_{i,\max}).
 \label{eq:positive-capacity-domain}
\end{equation}
The subsequent Riemannian construction is restricted to this open region.

\subsection{Actuator-Capacity Metric}

Define the actuator-capacity metric and co-metric by
\begin{align}
 G(v)
 &=\diag\bigl(\alpha_1(v_1)^{-2},\ldots,\alpha_n(v_n)^{-2}\bigr),
 \label{eq:metric}\\
 C(v)
 &=G(v)^{-1}
 =\diag\bigl(\alpha_1(v_1)^2,\ldots,\alpha_n(v_n)^2\bigr).
 \label{eq:cometric}
\end{align}
Because every $\alpha_i$ is positive on $\mathcal E$, $G(v)$ is positive definite and defines a Riemannian metric. For a tangent vector $\xi\in T_v\mathcal E$,
\begin{equation}
 g_v(\xi,\xi)
 =\xi^\T G(v)\xi
 =\sum_{i=1}^n\left(\frac{\xi_i}{\alpha_i(v_i)}\right)^2.
 \label{eq:normalized-effort}
\end{equation}
The metric therefore measures a quadratic aggregate of rotor accelerations normalized by their local symmetric capacities.

\begin{remark}[Model scope]
Current and voltage limits, back electromotive force, active or passive braking, battery variation, thermal derating, and an embedded speed controller can modify the actual directional interval. The construction above is exact for \eqref{eq:rotor-dynamics}. A different certified centered capacity function can replace \eqref{eq:sac} without changing the local constructions in Secs.~\ref{sec:local-geometry} and \ref{sec:fiberwise-selection}.
\end{remark}

\section{Capacity-Aware Local Task-Rate Geometry}
\label{sec:local-geometry}

Restrict the task map to
\begin{equation}
 \Ereg
 =\mathcal E\cap\Vspace_{\mathrm{reg}}.
 \label{eq:regular-capacity-region}
\end{equation}
At $v\in\Ereg$, define the vertical space
\begin{equation}
 \mathcal V_v=\ker J_f(v)=T_v\Fibre_{f(v)}
 \label{eq:vertical-space}
\end{equation}
and its $G(v)$-orthogonal complement
\begin{equation}
 \mathcal H_v=\mathcal V_v^{\perp_G}.
 \label{eq:horizontal-space}
\end{equation}
The horizontal space contains the actuator velocities that change the task without a metric-orthogonal internal component.

\begin{proposition}
For every $v\in\Ereg$,
\begin{equation}
    \mathcal H_v
    =
    \operatorname{im}\!\left(
        C(v)J_f(v)^\top
    \right).
    \label{eq:horizontal-range}
\end{equation}
\end{proposition}

\begin{proof}
For any \(z\in\ker J_f(v)\) and \(y\in\mathbb R^m\),
\(
    z^\top G(v)C(v)J_f(v)^\top y
    =
    z^\top J_f(v)^\top y
    =
    0.
\)
Therefore,
\(
    \operatorname{im}\!\left(
        C(v)J_f(v)^\top
    \right)
    \subseteq
    \mathcal H_v.
\)
Moreover, \(J_f(v)\) has full row rank, so
\(
    \operatorname{rank}\!\left(
        C(v)J_f(v)^\top
    \right)
    =
    m.
\)
Since \(\dim\mathcal H_v=m\), the two spaces coincide.
\end{proof}

\subsection{Horizontal Minimum-Effort Lift}

Given a desired wrench rate $\dot w\in T_{f(v)}\Wspace$, consider
\begin{equation}
 \min_{\dot v\in T_v\mathcal E}
 \frac12\dot v^\T G(v)\dot v
 \quad\text{subject to}\quad
 J_f(v)\dot v=\dot w.
 \label{eq:local-lift-problem}
\end{equation}
Define the local task-rate capacity matrix
\begin{equation}
 M(v)=J_f(v)C(v)J_f(v)^\T.
 \label{eq:capacity-matrix}
\end{equation}
It is symmetric positive definite on $\Ereg$.

\begin{theorem}[Capacity-horizontal lift]
For every $v\in\Ereg$ and every $\dot w$, problem \eqref{eq:local-lift-problem} has the unique solution
\begin{equation}
 \dot v_H
 =C(v)J_f(v)^\T M(v)^{-1}\dot w.
 \label{eq:horizontal-lift}
\end{equation}
This solution belongs to $\mathcal H_v$.
\end{theorem}

\begin{proof}
The objective is strictly convex because $G(v)$ is positive definite, and the affine constraint is feasible because $J_f(v)$ has full row rank. The stationarity condition for
\begin{equation*}
 \frac12\dot v^\T G(v)\dot v-\lambda^\T(J_f(v)\dot v-\dot w)
\end{equation*}
is $G(v)\dot v-J_f(v)^\T\lambda=0$. Hence $\dot v=C(v)J_f(v)^\T\lambda$. Applying the constraint gives $M(v)\lambda=\dot w$, which yields \eqref{eq:horizontal-lift}. Membership in $\mathcal H_v$ follows from \eqref{eq:horizontal-range}.
\end{proof}

Equation \eqref{eq:horizontal-lift} is the state-dependent weighted minimum-norm solution at a fixed actuator state. Its geometric interpretation is the unique capacity-horizontal lift. The outer operating-point selection in Sec.~\ref{sec:fiberwise-selection} instead compares different finite states on the same nonlinear fiber.

\subsection{Task-Rate Capability Ellipsoid}

\begin{proposition}
The minimum normalized actuator effort required to realize $\dot w$ is
\begin{equation}
 \min_{\dot v:J_f(v)\dot v=\dot w}
 \dot v^\T G(v)\dot v
 =\dot w^\T M(v)^{-1}\dot w.
 \label{eq:min-effort-identity}
\end{equation}
Moreover, the image of the actuator unit ball
\begin{equation}
 \mathcal B_v=\{\dot v:\dot v^\T G(v)\dot v\leq1\}
 \label{eq:actuator-unit-ball}
\end{equation}
through $J_f(v)$ is the task-rate ellipsoid
\begin{equation}
 \mathcal E_v^w
 =\{\dot w:\dot w^\T M(v)^{-1}\dot w\leq1\}.
 \label{eq:task-ellipsoid}
\end{equation}
\end{proposition}

\begin{proof}
Substitution of \eqref{eq:horizontal-lift} gives \eqref{eq:min-effort-identity}. A task rate belongs to the image of $\mathcal B_v$ if and only if at least one compatible actuator velocity has metric norm no greater than one. By \eqref{eq:min-effort-identity}, this is equivalent to the inequality in \eqref{eq:task-ellipsoid}.
\end{proof}

Let $\omega_m$ be the Euclidean volume of the unit ball in $\R^m$. The ellipsoid volume is
\begin{equation}
 \Vol(\mathcal E_v^w)
 =\omega_m\sqrt{\det M(v)}.
 \label{eq:ellipsoid-volume}
\end{equation}
We therefore define the DAAM indexes:
\begin{align}
 \rho(v)&=\sqrt{\det M(v)},
 \label{eq:volume-proxy}\\
 \ell(v)&=\log\rho(v)=\frac12\log\det M(v).
 \label{eq:log-volume}
\end{align}

The inverse $M(v)^{-1}$ defines an inner product on the task tangent space for each fixed $v\in\Ereg$. It is \emph{state attached}: in general it varies along $f^{-1}(w)$ and therefore does not define a metric depending only on the task point $w$.

For the signed-quadratic map, direct substitution gives
\begin{equation}
 M(v)
 =4A\diag\bigl(v_1^2\alpha_1(v_1)^2,\ldots,
 v_n^2\alpha_n(v_n)^2\bigr)A^\T.
 \label{eq:explicit-capacity-matrix}
\end{equation}
Each rotor contribution contains two factors. The term $v_i^2$ captures the differential of the signed-quadratic thrust law and vanishes at zero speed. The term $\alpha_i(v_i)^2$ captures the remaining symmetric acceleration authority and vanishes at the capacity boundary. Their product has an interior maximum.

\begin{corollary}
For a single rotor contribution under \eqref{eq:sac}, the scalar factor
\begin{equation}
 h_i(v_i)=v_i^2\alpha_i(v_i)^2
 \label{eq:rotor-contribution}
\end{equation}
is maximized at $|v_i|=v_{i,\max}/\sqrt{3}$.
\end{corollary}

\begin{proof}
Apart from a positive constant, $h_i(v_i)=v_i^2(\bar\tau_i-b_iv_i^2)^2$. Differentiation gives
\begin{equation*}
 h_i'(v_i)
 =2v_i(\bar\tau_i-b_iv_i^2)(\bar\tau_i-3b_iv_i^2)/m_i^2.
\end{equation*}
The nonzero interior critical points satisfy $v_i^2=\bar\tau_i/(3b_i)$ and are maxima because $h_i$ vanishes at zero and at the capacity boundaries.
\end{proof}

\section{Fiberwise Capacity-Aware Operating-Point Selection}
\label{sec:fiberwise-selection}

Let $K\subset\Ereg$ be a nonempty compact set representing the regular operating region retained for selection. For a task value $w$, define
\begin{equation}
 K_w=K\cap f^{-1}(w).
 \label{eq:compact-fiber}
\end{equation}
Whenever $K_w$ is nonempty, define the capacity-maximizer correspondence
\begin{equation}
 \mathcal P_w
 =\argmax_{v\in K_w}\ell(v).
 \label{eq:maximizer-correspondence}
\end{equation}
The compact restriction avoids making an attainment claim on the open positive-capacity region.

\begin{proposition}[Existence]
If $K_w$ is nonempty, then $\mathcal P_w$ is nonempty and compact.
\end{proposition}

\begin{proof}
The matrix $M(v)$ is positive definite and continuous on $K\subset\Ereg$, so $\ell$ is continuous there. The set $K_w$ is a closed subset of compact $K$ and is therefore compact. The conclusion follows from the extreme-value theorem.
\end{proof}

\subsection{Task-Coordinate Invariance}

Consider a smooth change of task coordinates $\widetilde w=\psi(w)$ with nonsingular Jacobian $T(w)=D\psi(w)$. The transformed task-map Jacobian is $\widetilde J_f(v)=T(f(v))J_f(v)$, and thus
\begin{equation}
 \widetilde M(v)
 =T(f(v))M(v)T(f(v))^\T.
 \label{eq:capacity-transform}
\end{equation}

\begin{proposition}[Fiberwise task-coordinate invariance]
For every $w$ and every $v\in K_w$,
\begin{equation}
 \widetilde\rho(v)
 =|\det T(w)|\rho(v).
 \label{eq:volume-transform}
\end{equation}
Consequently, the maximizing set $\mathcal P_w$ is unchanged by the task-coordinate change.
\end{proposition}

\begin{proof}
Taking determinants in \eqref{eq:capacity-transform} gives $\det\widetilde M(v)=\det T(f(v))^2\det M(v)$. On $K_w$, $f(v)=w$, so the positive factor $|\det T(w)|$ is constant with respect to the optimization variable.
\end{proof}

\subsection{Interior First-Order Conditions}

The following calculation is included to support implementation and interpretation of interior stationary points. It is restricted to $v\in\Ereg$ with $v_i\neq0$ for every $i$.

Let $a_i$ now denote the $i$th column of $A$ and define $h_i(v_i)=v_i^2\alpha_i(v_i)^2$. From \eqref{eq:explicit-capacity-matrix},
\begin{equation}
 \frac{\partial M}{\partial v_i}
 =4h_i'(v_i)a_ia_i^\T.
 \label{eq:M-derivative}
\end{equation}
Hence
\begin{align}
 \frac{\partial\ell}{\partial v_i}
 &=\frac12\operatorname{tr}
 \left(M(v)^{-1}\frac{\partial M}{\partial v_i}\right)
 \nonumber\\
 &=2h_i'(v_i)a_i^\T M(v)^{-1}a_i
 \nonumber\\
 &=4v_i\alpha_i(v_i)
 \left(\alpha_i(v_i)-\frac{2b_iv_i^2}{m_i}\right)
 a_i^\T M(v)^{-1}a_i.
 \label{eq:gradient}
\end{align}
Equivalently,
\begin{equation}
 \frac{\partial\ell}{\partial v_i}
 =\frac{4v_i}{m_i^2}
 (\bar\tau_i-b_iv_i^2)
 (\bar\tau_i-3b_iv_i^2)
 a_i^\T M(v)^{-1}a_i.
 \label{eq:gradient-expanded}
\end{equation}

For the Lagrangian $L(v,\lambda)=\ell(v)-\lambda^\T(f(v)-w)$, interior stationarity requires
\begin{equation}
 \frac{\partial\ell}{\partial v_i}
 -2|v_i|a_i^\T\lambda=0.
 \label{eq:kkt-component}
\end{equation}
The factor $|v_i|$ in \eqref{eq:kkt-component} is essential. If $v_i\neq0$, division by $2|v_i|$ yields the equivalent chamberwise relation
\begin{equation}
 2\operatorname{sgn}(v_i)\alpha_i(v_i)
 \left(\alpha_i(v_i)-\frac{2b_iv_i^2}{m_i}\right)
 a_i^\T M(v)^{-1}a_i
 =a_i^\T\lambda.
 \label{eq:kkt-cancelled}
\end{equation}
Neither expression applies as a first-order characterization of candidates with $v_i=0$ after division. Such points must be checked separately.

\section{Canonical Illustrations}
\label{sec:illustrations}

This section uses low-dimensional systems to visualize the SAC-induced
actuator geometry, DAAM variation along allocation fibers, and fiberwise
operating-point selection.

\subsection{Two Rotors and One Scalar Task}
\label{subsec:two-one}

Consider
\begin{equation}
 f(v_1,v_2)
 =
 A_1v_1|v_1|+A_2v_2|v_2|,
 \label{eq:dipole-map}
\end{equation}
where \(A_1,A_2\neq 0\). The task-rate capability matrix is scalar:
\begin{equation}
 M(v)
 =
 4\sum_{i=1}^2
 A_i^2v_i^2\alpha_i(v_i)^2.
 \label{eq:dipole-capacity}
\end{equation}
Hence, \(\sqrt{M(v)}\) is the radius of the bidirectional task-rate interval
generated by the actuator-capacity unit ball. For a scalar task, maximizing
the DAAM index on a fiber therefore maximizes the bidirectional task-rate
authority.

\begin{figure}[t]
  \centering
  \includegraphics[width=0.94\linewidth]{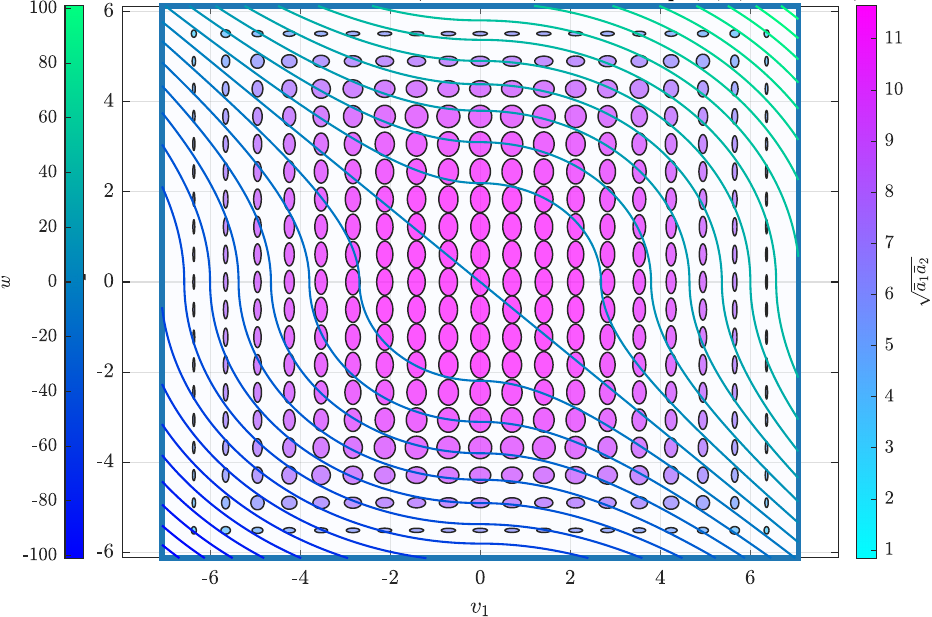}
  \caption{Rotor-speed space \((v_1,v_2)\), positive-capacity region
  \(\mathcal E\), unit balls of the SAC-induced actuator metric \(G(v)\), and
  representative allocation fibers. The ellipses contract along directions
  in which the corresponding SAC approaches zero and are uniformly scaled
  for readability.}
  \label{fig:metric-fibers}
\end{figure}

Figure~\ref{fig:metric-fibers} shows the state dependence of the
actuator-capacity unit ball. Approaching a capacity boundary reduces the
bidirectional acceleration authority along the corresponding actuator
direction. The allocation fibers make it possible to compare this local
geometry among states generating the same scalar task value.

\begin{figure}[t]
  \centering
  \includegraphics[width=0.94\linewidth]
  {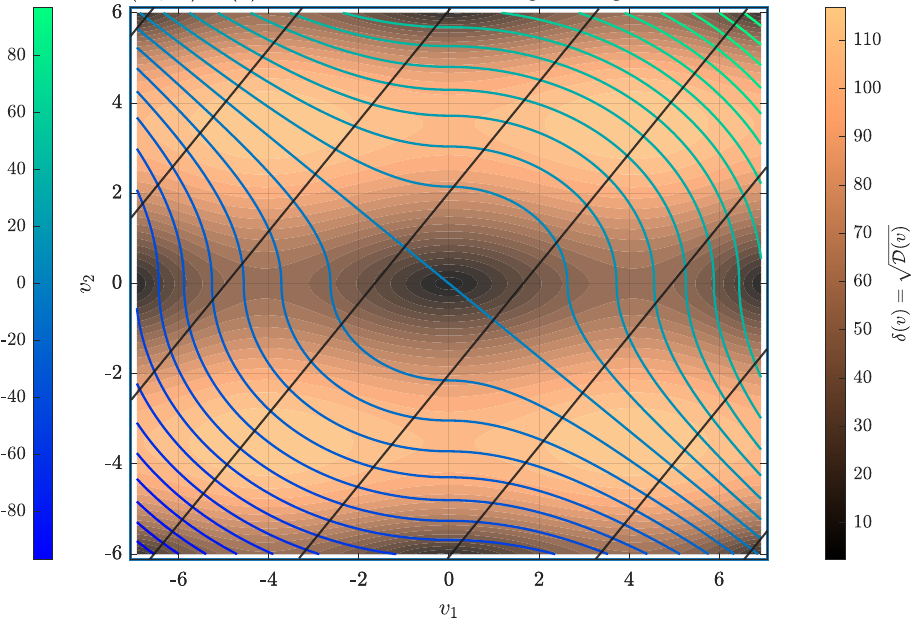}
  \caption{Two-rotor DAAM landscape, representative allocation fibers, and
  straight sampling slices in rotor-speed space. The slices provide
  one-dimensional traces of the landscape and are not generally normal to
  the nonlinear fibers.}
  \label{fig:capacity-landscape}
\end{figure}

The landscape in Fig.~\ref{fig:capacity-landscape} reflects the two factors
in \eqref{eq:explicit-capacity-matrix}. Near zero rotor speed, the differential
of the signed-quadratic actuation map is weak. Near the positive-capacity
boundaries, aerodynamic drag leaves little acceleration authority guaranteed
in both directions. The largest DAAM values consequently occur at
intermediate rotor speeds, subject to the combinations permitted by each
allocation fiber.

\begin{figure}[t]
  \centering
  \includegraphics[width=0.94\linewidth]
  {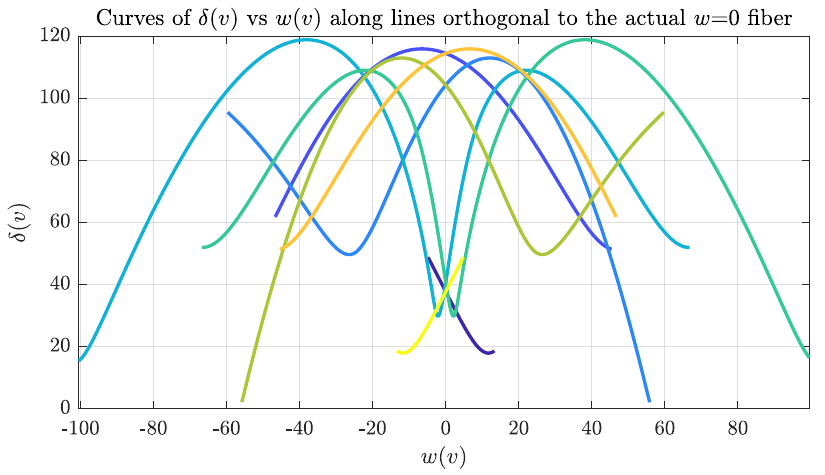}
  \caption{DAAM index versus task value along the straight sampling slices in
  Fig.~\ref{fig:capacity-landscape}. The upper envelope refers to the sampled
  slices, whereas the operating-point criterion maximizes DAAM over each
  complete allocation fiber.}
  \label{fig:capacity-traces}
\end{figure}

Figure~\ref{fig:capacity-traces} provides a one-dimensional representation of
the DAAM landscape along the selected slices. These traces differ from the
fiberwise optimization in \eqref{eq:maximizer-correspondence}, which searches
over all task-equivalent actuator states.

\subsection{Two-Rotor Parameter Study}
\label{subsec:two-parameter}

Figure~\ref{fig:parameter-gallery} examines the dependence of the fiberwise
DAAM maximizers on the static actuation coefficients and the rotor parameters
entering the SAC. Starting from a symmetric baseline, each panel varies one
parameter. The background represents the DAAM index, the thin red curves are
allocation fibers, and the yellow markers identify numerically computed
fiberwise maximizers within the displayed compact operating domain.

\begin{table}[t]
\centering
\caption{Two-rotor parameter variations in
Fig.~\ref{fig:parameter-gallery}.}
\label{tab:two-rotor-parameters}
\footnotesize
\begin{tabular}{lll}
\toprule
Panel & Case & Change from baseline \\
\midrule
\ref{fig:caseA_small_a1}   & Smaller coefficient 1 & \(A_1=0.7\) \\
\ref{fig:caseA_balanced}   & Symmetric baseline    & None \\
\ref{fig:caseA_small_a2}   & Smaller coefficient 2 & \(A_2=0.7\) \\
\ref{fig:caseA_tiny_m1}    & Smaller inertia       & \(m_1=0.005\) \\
\ref{fig:caseA_tiny_tau1}  & Smaller torque limit  & \(\bar\tau_1=0.1\) \\
\ref{fig:caseA_tiny_b1}    & Smaller drag          & \(b_1=0.01\) \\
\bottomrule
\end{tabular}
\end{table}

\begin{figure*}[t]
    \centering
    \begin{subfigure}{0.32\textwidth}
        \centering
        \includegraphics[width=\linewidth]{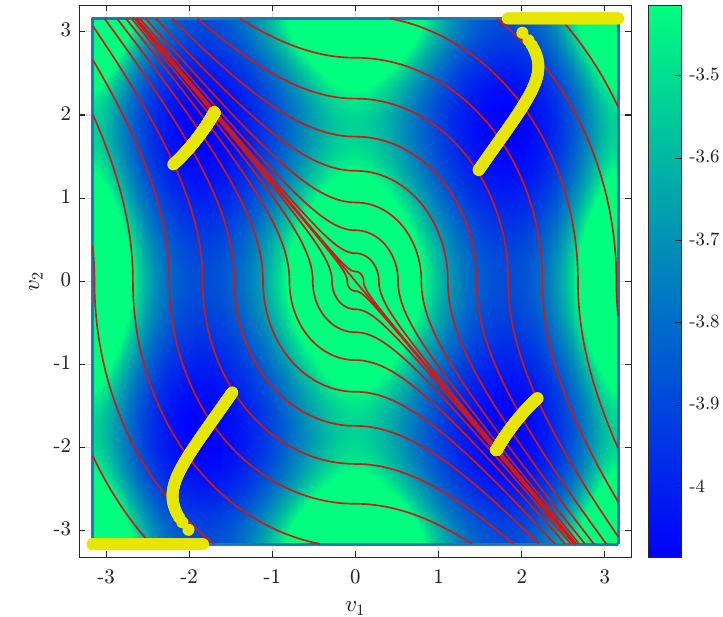}
        \subcaption{\(A_1=0.7\)}
        \label{fig:caseA_small_a1}
    \end{subfigure}\hfill
    \begin{subfigure}{0.32\textwidth}
        \centering
        \includegraphics[width=\linewidth]{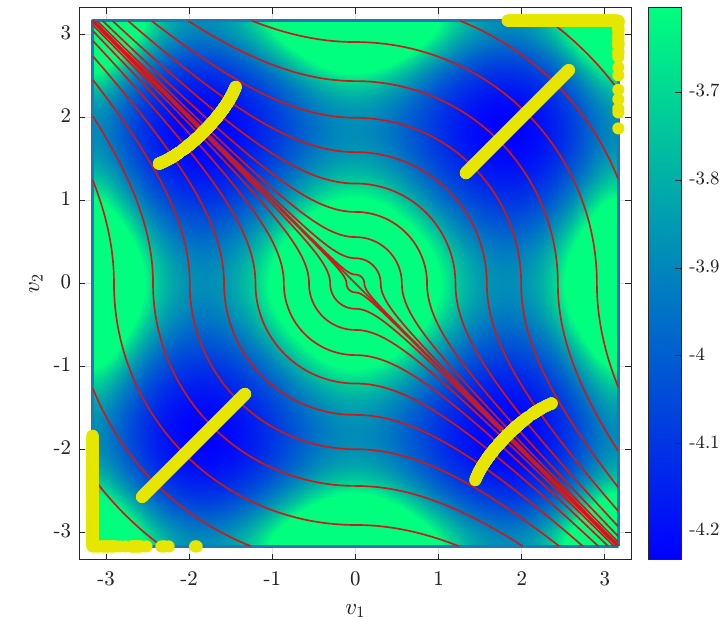}
        \subcaption{Symmetric baseline}
        \label{fig:caseA_balanced}
    \end{subfigure}\hfill
    \begin{subfigure}{0.32\textwidth}
        \centering
        \includegraphics[width=\linewidth]{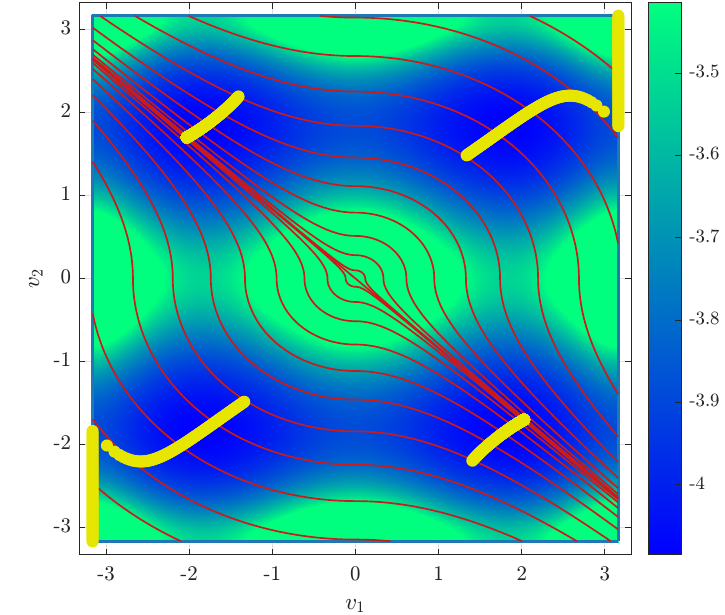}
        \subcaption{\(A_2=0.7\)}
        \label{fig:caseA_small_a2}
    \end{subfigure}

    \vspace{0.6em}

    \begin{subfigure}{0.32\textwidth}
        \centering
        \includegraphics[width=\linewidth]
        {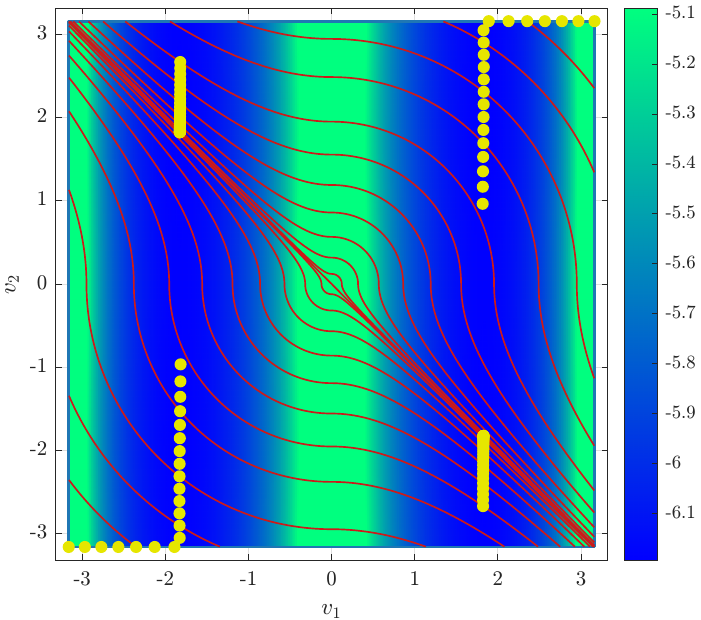}
        \subcaption{\(m_1=0.005\)}
        \label{fig:caseA_tiny_m1}
    \end{subfigure}\hfill
    \begin{subfigure}{0.15\textwidth}
        \centering
        \includegraphics[width=\linewidth]
        {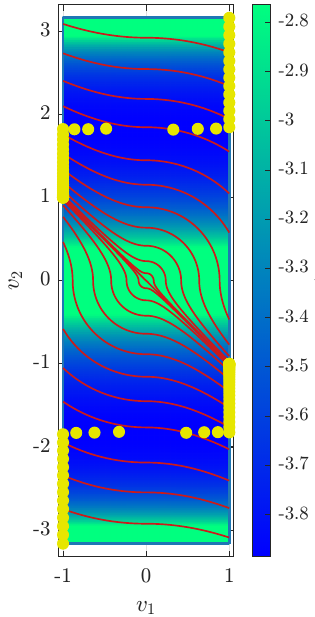}
        \subcaption{\(\bar\tau_1=0.1\)}
        \label{fig:caseA_tiny_tau1}
    \end{subfigure}\hfill
    \begin{subfigure}{0.48\textwidth}
        \centering
        \includegraphics[width=\linewidth]
        {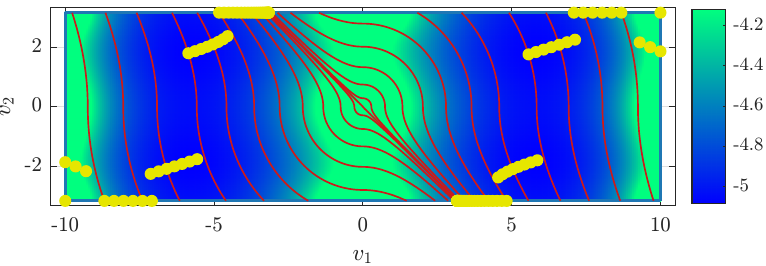}
        \subcaption{\(b_1=0.01\)}
        \label{fig:caseA_tiny_b1}
    \end{subfigure}

    \caption{Two-rotor scalar-task DAAM landscapes and fiberwise maximizing
    candidates. Individual parameters are varied from the symmetric baseline
    \(A_1=A_2=1\), \(b_1=b_2=0.1\), \(m_1=m_2=0.05\), and
    \(\bar\tau_1=\bar\tau_2=1\). Thin red curves are allocation fibers, and
    yellow markers identify numerically computed DAAM maximizers on the
    displayed compact operating domain.}
    \label{fig:parameter-gallery}
\end{figure*}

The first row varies the static actuation coefficients. The baseline in
Fig.~\ref{fig:caseA_balanced} produces symmetric DAAM and fiber patterns.
Reducing \(A_1\) or \(A_2\) changes both the allocation fibers and the
corresponding rotor contribution to the task-rate capability matrix. The
cases \(A_1=0.7\) and \(A_2=0.7\) consequently exhibit exchanged-coordinate
versions of the same asymmetric structure.

The second row varies parameters entering the SAC. Reducing \(m_1\) increases
the normalized capacity contribution of the first rotor throughout its
admissible speed range and substantially reorganizes the maximizing
candidates. Reducing \(\bar\tau_1\) contracts the positive-capacity interval
of the first rotor, narrowing the domain along \(v_1\). Reducing \(b_1\)
expands that interval and produces the elongated domain in
Fig.~\ref{fig:caseA_tiny_b1}.

The yellow candidates form separated pieces whose locations and shapes depend
on the task value and propulsion parameters. Some pieces also meet the
boundary of the displayed domain. Thus, even in this two-rotor setting, the
fiberwise optimizer is naturally represented by the correspondence
\(\mathcal P_w\), without presuming a unique smooth branch.

\subsection{Three Rotors and One Scalar Task}
\label{subsec:three-one}

For three rotors and one scalar task,
\begin{equation}
 f(v)
 =
 \sum_{i=1}^3 A_iv_i|v_i|,
 \qquad
 M(v)
 =
 4\sum_{i=1}^3 A_i^2v_i^2\alpha_i(v_i)^2.
 \label{eq:three-one}
\end{equation}
Every regular allocation fiber is locally a two-dimensional surface in the
three-dimensional actuator space. The example in
Fig.~\ref{fig:three-one} uses the symmetric parameters
\(A_1=A_2=A_3=1\), \(b_i=0.1\), \(m_i=0.05\), and
\(\bar\tau_i=1\) for all \(i\).

\begin{figure*}[t]
  \centering
  \begin{subfigure}{0.33\textwidth}
    \centering
    \includegraphics[width=\linewidth]{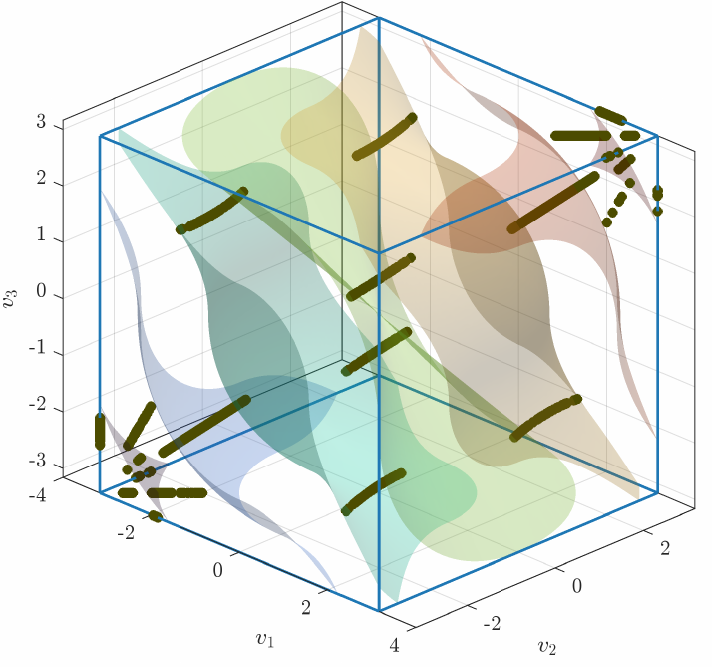}
    \subcaption{View 1}
    \label{fig:case3D_v1}
  \end{subfigure}\hfill
  \begin{subfigure}{0.32\textwidth}
    \centering
    \includegraphics[width=\linewidth]{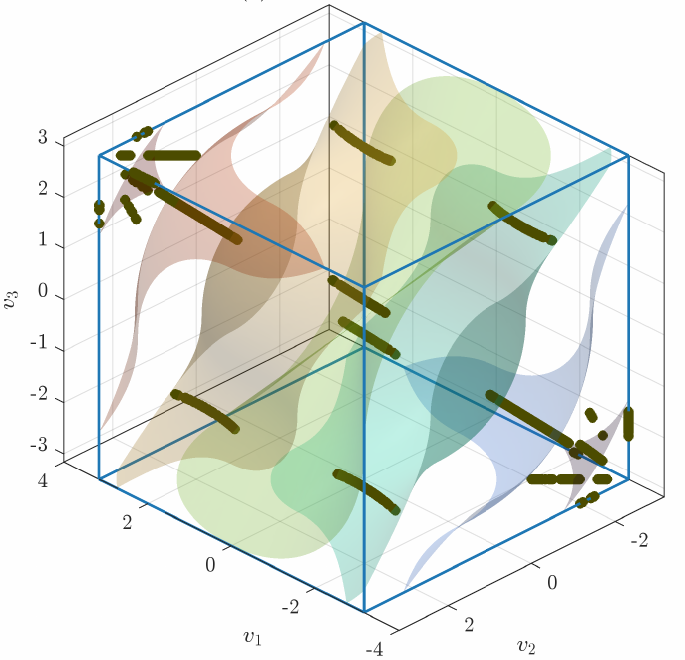}
    \subcaption{View 2}
    \label{fig:case3D_v3}
  \end{subfigure}\hfill
  \begin{subfigure}{0.32\textwidth}
    \centering
    \includegraphics[width=\linewidth]{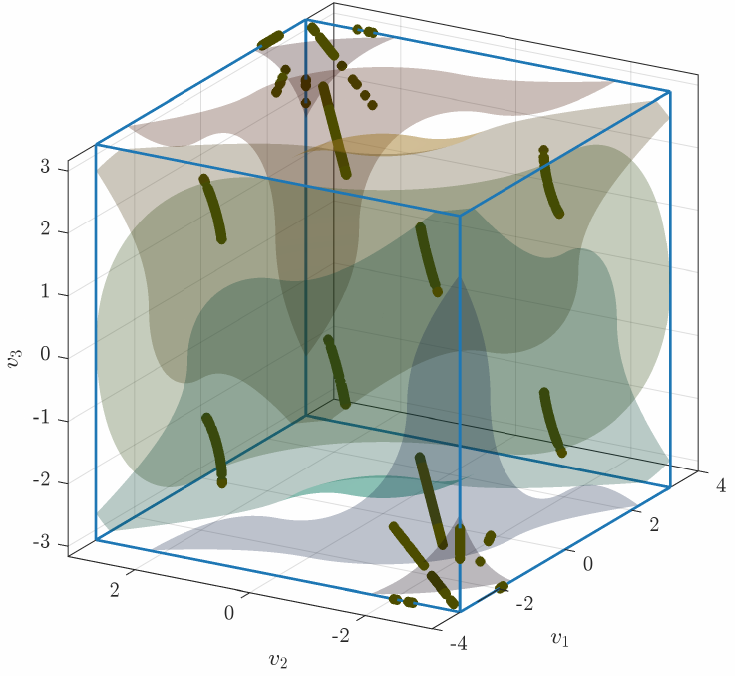}
    \subcaption{View 3}
    \label{fig:case3D_v4}
  \end{subfigure}
  \caption{Three-rotor, one-task allocation fibers and numerically computed
  fiberwise DAAM maximizers shown from three viewpoints. The translucent
  surfaces represent fibers associated with selected scalar task values,
  while the yellow markers identify maximizing candidates on those surfaces.
  The wireframe box indicates the compact operating domain.}
  \label{fig:three-one}
\end{figure*}

The third rotor increases the fiber dimension from one to two. Each
translucent surface in Fig.~\ref{fig:three-one} contains actuator states
generating one selected scalar task value, and the yellow markers identify
the computed DAAM maximizers on that surface. The candidates occur in several
regions of the actuator space rather than along a single branch.

This example makes the additional task-preserving freedom visible. DAAM must
now be compared over an entire surface, allowing redistribution among the
three rotor contributions along two independent tangent directions. The
definition of the maximizing correspondence \(\mathcal P_w\) remains
unchanged, although its computation and spatial structure become
higher-dimensional.

\subsection{Three Rotors and Two Tasks}
\label{subsec:three-two}

For \(n=3\) and \(m=2\), every regular allocation fiber is locally one
dimensional. The task-rate capability ellipsoid is two dimensional, and the
DAAM index \(\sqrt{\det M(v)}\) is proportional to its area. Fiberwise DAAM
maximization therefore selects the task-equivalent actuator states providing
the largest local capability area.

\begin{figure*}[t]
  \centering
  \begin{subfigure}{0.32\textwidth}
    \centering
    \includegraphics[width=\linewidth]{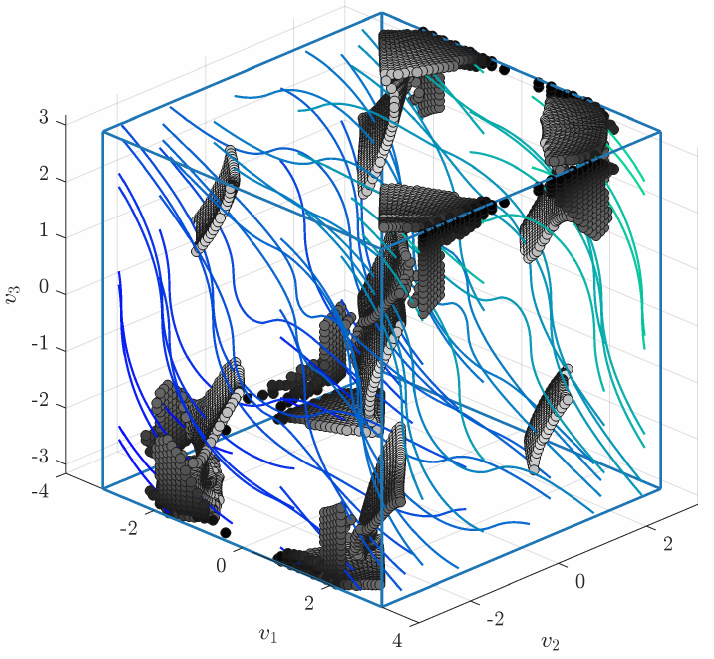}
    \subcaption{View 1}
    \label{fig:case3D2_v1}
  \end{subfigure}\hfill
  \begin{subfigure}{0.32\textwidth}
    \centering
    \includegraphics[width=\linewidth]{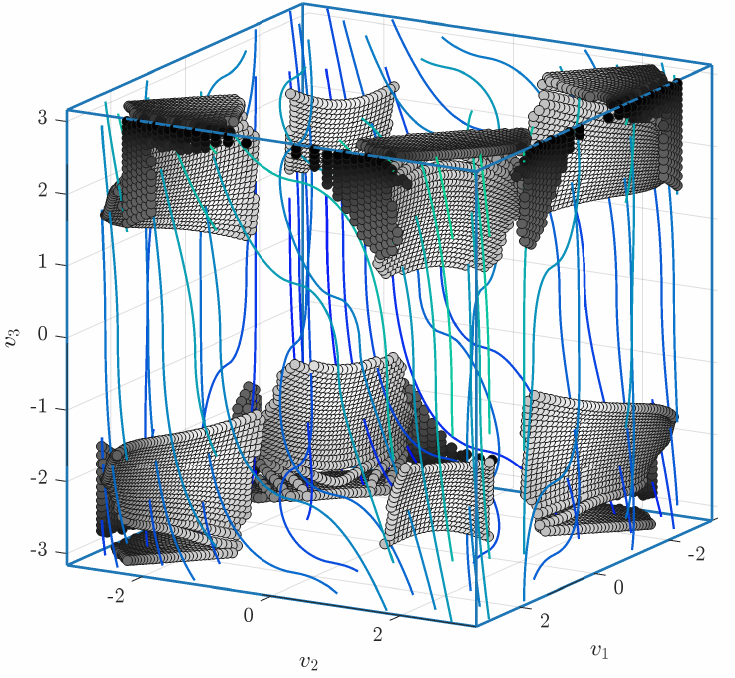}
    \subcaption{View 2}
    \label{fig:case3D2_v2}
  \end{subfigure}\hfill
  \begin{subfigure}{0.32\textwidth}
    \centering
    \includegraphics[width=\linewidth]{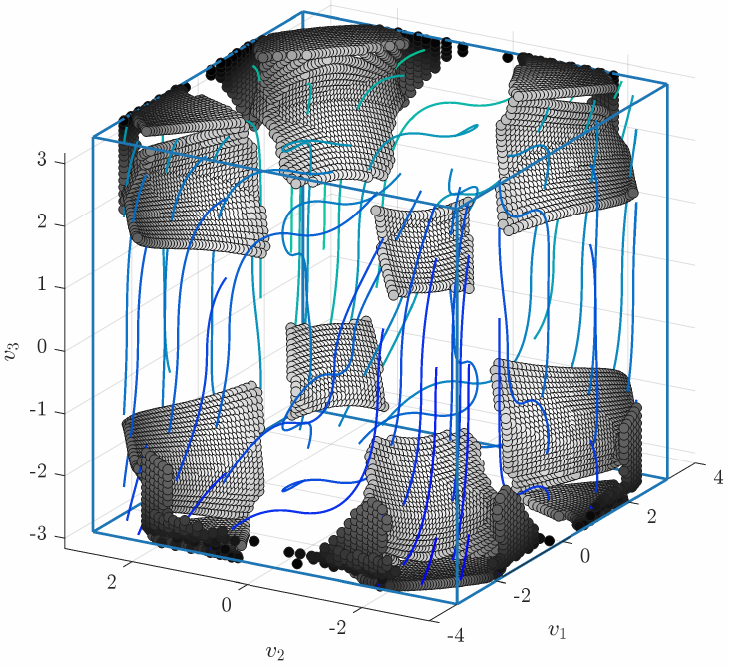}
    \subcaption{View 3}
    \label{fig:case3D2_v4}
  \end{subfigure}
  \caption{Three-rotor, two-task allocation fibers and numerically computed
  fiberwise DAAM maximizers shown from three viewpoints. Colored curves
  represent fibers associated with sampled task wrenches, while gray markers
  identify their maximizing candidates. Collectively, the candidates form
  multiple separated pieces in the actuator space.}
  \label{fig:three-two}
\end{figure*}

As the two-dimensional task wrench varies, the corresponding one-dimensional
fibers sweep through the three-dimensional actuator space. The computed
maximizers collected over the sampled task region form multiple curved and
separated pieces rather than a single smooth section. This global structure
can therefore be multipiece even though every individual regular fiber is
locally one dimensional.

\subsection{Cross-Case Synthesis}

The examples show how the fiberwise DAAM selection changes with the rotor
parameters and with the dimensions of the actuator and task spaces. The
two-rotor study exposes its parameter dependence in the plane. Adding a third
rotor turns each scalar-task fiber into a surface with two task-preserving
internal directions. Increasing the task dimension to two restores
one-dimensional fibers, while their maximizers, collected across the sampled
task region, form a multidimensional collection in actuator space.

In all cases, the selected states depend jointly on the signed-quadratic
actuation differential and the SAC. The computations exhibit separated
maximizing pieces, possible multiplicity, and encounters with the selected
operating-domain boundary. These features support the set-valued formulation
of \(\mathcal P_w\) and motivate the explicit treatment of continuity and
actuator dynamics when constructing an executable allocation section.

\section{Low-Dimensional Use Case: Continuous DAAM Allocation and Saturation-Limited Force Tracking} \label{sec:continuous-allocation}

The fiberwise correspondence in
\eqref{eq:maximizer-correspondence} identifies actuator states with maximum
local task-rate capability, but independent maximization on each fiber need
not define a continuous allocation section. This section presents a
low-dimensional use case showing how the DAAM field and the SAC-induced
geometry can inform the construction of such a section. The two-rotor scalar
task makes the allocation fibers and competing sections directly visible,
while simulation of the saturated rotor dynamics isolates their effect on
force tracking.

\subsection{Continuous Torque-Feasible DAAM Section}
\label{subsec:continuous-daam-section}

Consider two cooperative rotors operating at \(v_1,v_2>0\) and generating the
scalar force
\begin{equation}
 w=f(v)=A_1v_1^2+A_2v_2^2.
 \label{eq:cooperative-two-rotor-map}
\end{equation}
For a compact force interval
\([w_{\mathrm L},w_{\mathrm U}]\) contained in the regular
positive-capacity region, an allocation section is a continuously
differentiable map
\(s:[w_{\mathrm L},w_{\mathrm U}]\to\mathbb R_{>0}^2\)
satisfying \(f(s(w))=w\). Along the lifted trajectory
\(v(t)=s(w(t))\), exact section following requires
\begin{equation}
 \tau_i
 =
 b_is_i(w)^2+m_is_i'(w)\dot w.
 \label{eq:section-following-torque}
\end{equation}
This expression uses the complete directional rotor dynamics and therefore
retains the asymmetric effect of aerodynamic drag.

Given a symmetric force-rate envelope
\(|\dot w|\leq\dot w_{\mathrm d}\) and a torque-use factor
\(\eta\in(0,1]\), a continuous DAAM section is computed from
\begin{equation}
 \begin{aligned}
 \max_{s\in C^1}\quad
 &\int_{w_{\mathrm L}}^{w_{\mathrm U}}\ell(s(w))\,\dd w\\
 \text{subject to}\quad
 &f(s(w))=w,\\
 &-\eta\bar\tau_i
 \leq b_is_i(w)^2+m_is_i'(w)q
 \leq\eta\bar\tau_i,\\
 &q\in\{-\dot w_{\mathrm d},\dot w_{\mathrm d}\},
 \quad i\in\{1,2\}.
 \end{aligned}
 \label{eq:continuous-daam-section}
\end{equation}
The objective favors sections traversing regions of high DAAM, while the
constraints ensure torque feasibility at both extrema of the prescribed
force-rate interval. The factor \(\eta\) reserves the remaining fraction of
the torque range for tracking correction.

The fiber constraint is satisfied identically by parameterizing the section
as
\begin{equation}
 s_1(w)=\sqrt{\frac{w}{A_1}}\cos\theta(w),
 \quad
 s_2(w)=\sqrt{\frac{w}{A_2}}\sin\theta(w).
 \label{eq:angular-section-parameterization}
\end{equation}
In the numerical implementation, \(\theta(w)\) is represented by a
shape-preserving piecewise-cubic interpolant. Its control values are optimized
while enforcing the inequalities in
\eqref{eq:continuous-daam-section} on a dense force grid.

Two comparison sections are considered. Independent maximization of DAAM on
each fiber provides the pointwise capability reference, although its
maximizing candidates need not form a continuous section. The unweighted
pseudoinverse is defined in the individual-force variables
\(u_i=A_iv_i^2\). Since \(u_1+u_2=w\), its minimum-Euclidean-norm solution is
\(u_1=u_2=w/2\), which gives
\begin{equation}
 s_{\mathrm{PI},i}(w)
 =
 \sqrt{\frac{w}{2A_i}}.
 \label{eq:pseudoinverse-section}
\end{equation}
This section depends on the static coefficients \(A_i\), but not on the
inertias, drag coefficients, or torque limits. It consequently assigns the
same individual forces to rotors having equal static effectiveness even when
their dynamic capabilities differ.

\subsection{Saturated Rotor-Dynamics Benchmark}
\label{subsec:tracking-benchmark}

Each allocation section \(s_j\) is followed using the same rotor-speed
controller,
\begin{equation}
 \tau_{i,\mathrm{cmd}}
 =
 b_iv_i^2
 +
 m_i\left[
 s_{j,i}'(w_d)\dot w_d
 +
 k_{v,i}\bigl(s_{j,i}(w_d)-v_i\bigr)
 \right],
 \label{eq:rotor-speed-controller}
\end{equation}
followed by the componentwise saturation
\[
 \tau_i
 =
 \operatorname{sat}_{[-\bar\tau_i,\bar\tau_i]}
 \bigl(\tau_{i,\mathrm{cmd}}\bigr).
\]
The first term in \eqref{eq:rotor-speed-controller} compensates aerodynamic
drag, the second follows the section tangent, and the third corrects the
rotor-speed error. The realized force is \(w=f(v)\).

The desired force is a reproducible band-limited multisine,
\begin{equation}
 w_d(t)
 =
 w_0
 +
 \sum_{r=1}^{N_h}
 c_r\sin(2\pi f_rt+\varphi_r),
 \label{eq:multisine-command}
\end{equation}
scaled to remain inside the section domain. Four frequency bands are used:
\(0.05\)--\(0.20\) Hz, \(0.20\)--\(0.50\) Hz,
\(0.50\)--\(1.00\) Hz, and \(1.00\)--\(1.80\) Hz. Each band comprises eight
seeded realizations with eight harmonic components. All simulations use the
same initial task value, controller gains, command realizations, torque
limits, and \(2\) ms integration step. Each simulation lasts \(15\) s, and
the first \(2\) s are excluded from the performance metrics.

Tracking performance is measured by the RMS force error normalized by the
standard deviation of the desired force. Torque saturation is measured by the
fraction of retained samples at which at least one unsaturated torque command
lies outside its admissible interval.

\subsection{Case I: Equal Static Coefficients and Unequal Propulsion Dynamics}
\label{subsec:dynamic-asymmetry-case}

The first case uses
\(
 A=(1,1),\;
 b=(0.2,0.1),\;
 m=(0.05,0.05),\;
 \bar\tau=(0.6,1).
\)
The physical cooperative force range is \([0,13]\), while the section is
constructed on \([0.39,11.05]\) using \(\eta=0.9\) and
\(\dot w_{\mathrm d}=0.15\). The equal static coefficients make this case
particularly informative: the pseudoinverse assigns equal individual forces
even though the first rotor has both larger drag and a smaller torque limit.

\begin{figure*}[t]
 \centering
 \begin{subfigure}{0.4\textwidth}
  \centering
  \includegraphics[width=\linewidth]
  {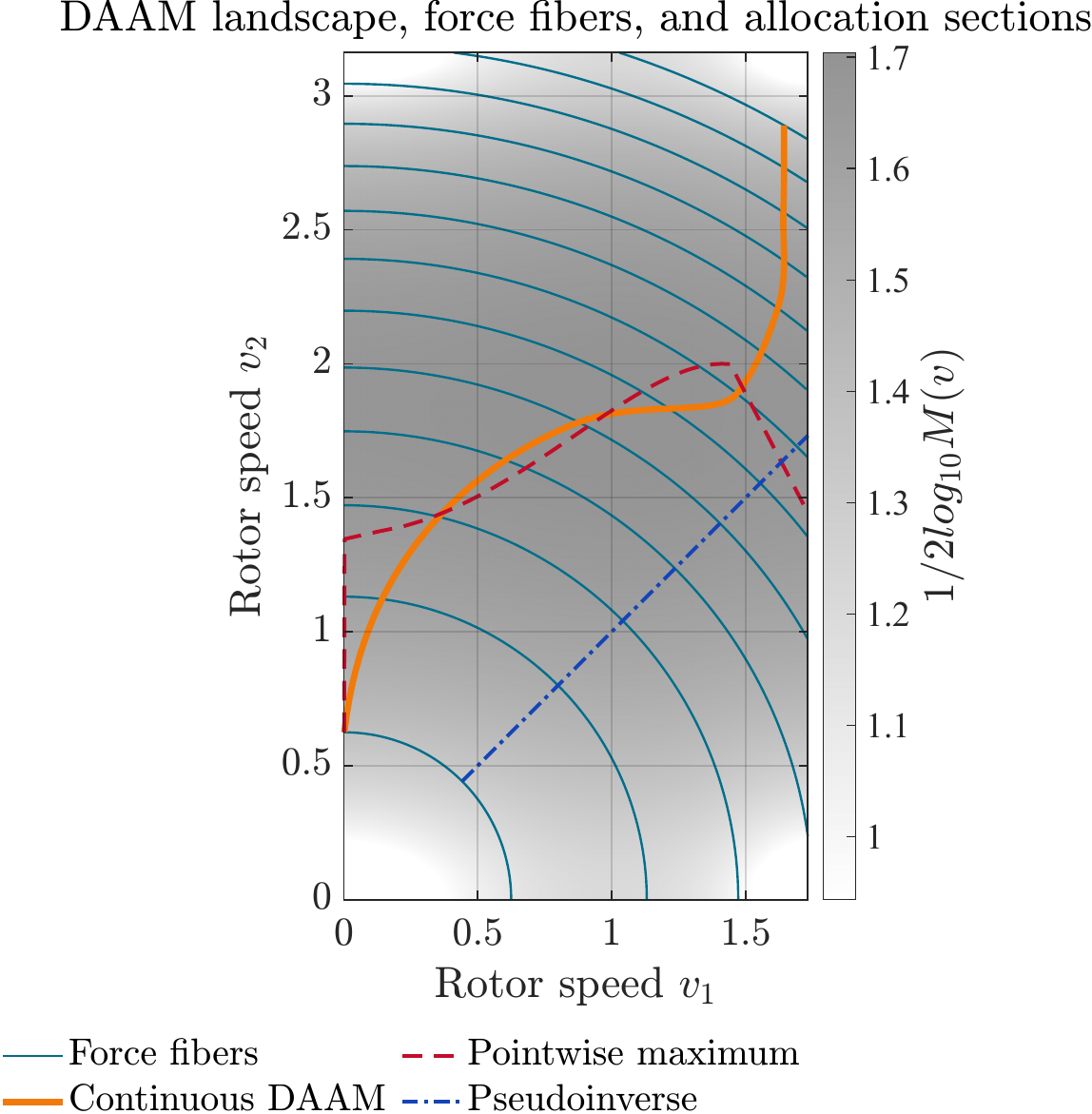}
  \subcaption{DAAM landscape, fibers, and allocation sections.}
  \label{fig:daam-case1-allocation}
 \end{subfigure}\hfill
 \begin{subfigure}{0.58\textwidth}
  \centering
  \includegraphics[width=\linewidth]
  {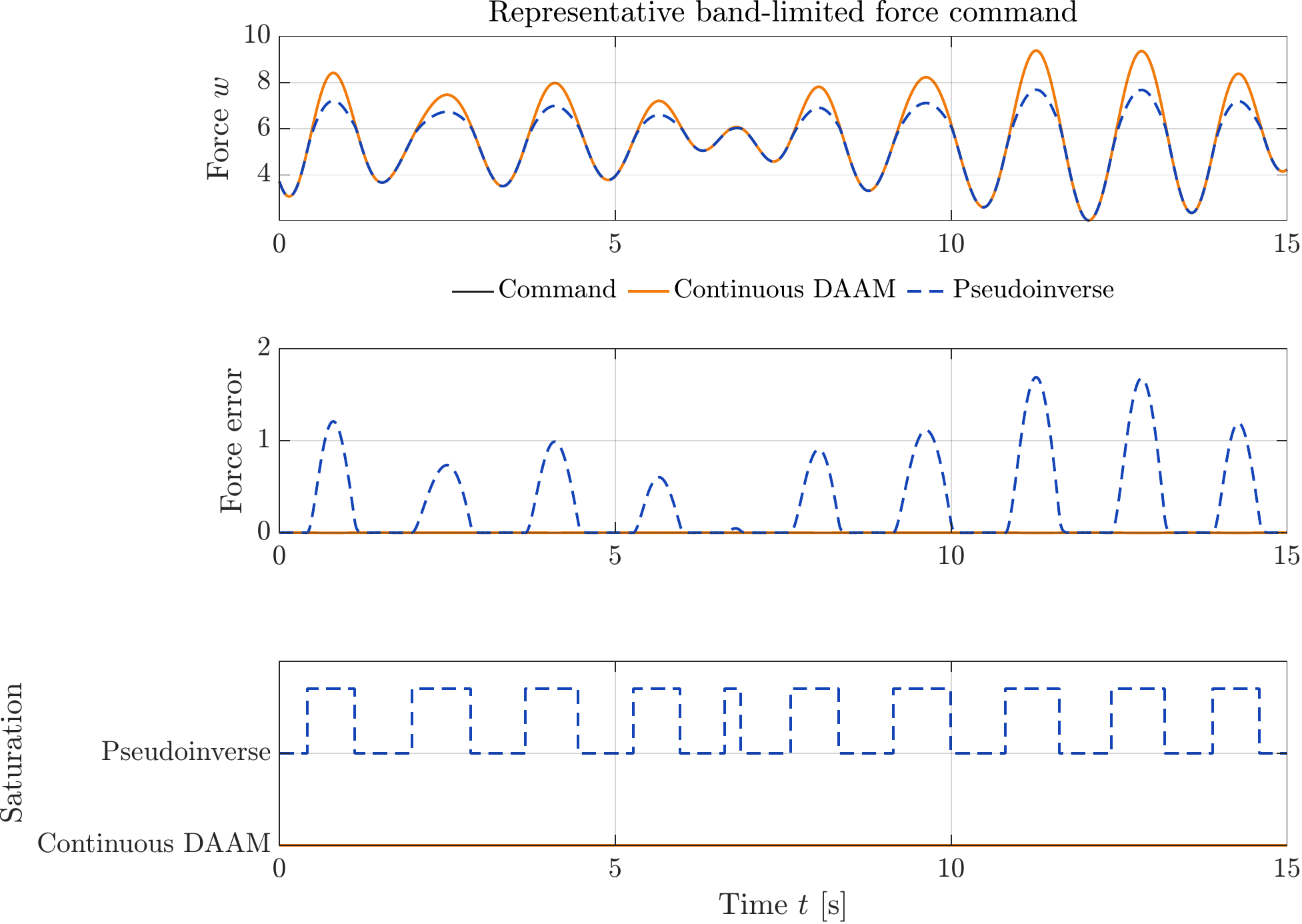}
  \subcaption{Representative force-tracking realization.}
  \label{fig:daam-case1-tracking}
 \end{subfigure}
 \caption{Allocation and saturated force tracking for Case~I. The pointwise
 DAAM maximizers provide a capability reference, while the continuous DAAM
 and pseudoinverse sections are followed dynamically.}
 \label{fig:daam-case1-main}
\end{figure*}

Figure~\ref{fig:daam-case1-allocation} shows that the optimized section
traverses the high-DAAM region while remaining continuous and
torque-feasible. It departs from the pointwise maximizing candidates where
following them would be incompatible with the section constraints. The mean
log-DAAM over the construction interval is \(3.7231\) for the continuous
DAAM section, \(3.7105\) for the pseudoinverse, and \(3.7787\) for the
pointwise reference. The maximum nominal torque utilization reaches the
prescribed value \(\eta=0.9\).

In the representative realization of
Fig.~\ref{fig:daam-case1-tracking}, the two executable sections track
similarly while the required torque remains admissible. Their responses
separate when the faster force variations activate the torque limits. The
same frequency dependence appears in the aggregate results of
Fig.~\ref{fig:daam-case1-bands}: differences are limited in the lowest band
and become more pronounced as the command bandwidth increases. In the
very-high-frequency band, the continuous DAAM section yields both a lower
normalized RMS force error and a smaller saturation fraction than the
pseudoinverse.

\begin{figure*}[t]
 \centering
 \begin{subfigure}{0.49\textwidth}
  \centering
  \includegraphics[width=\linewidth]
  {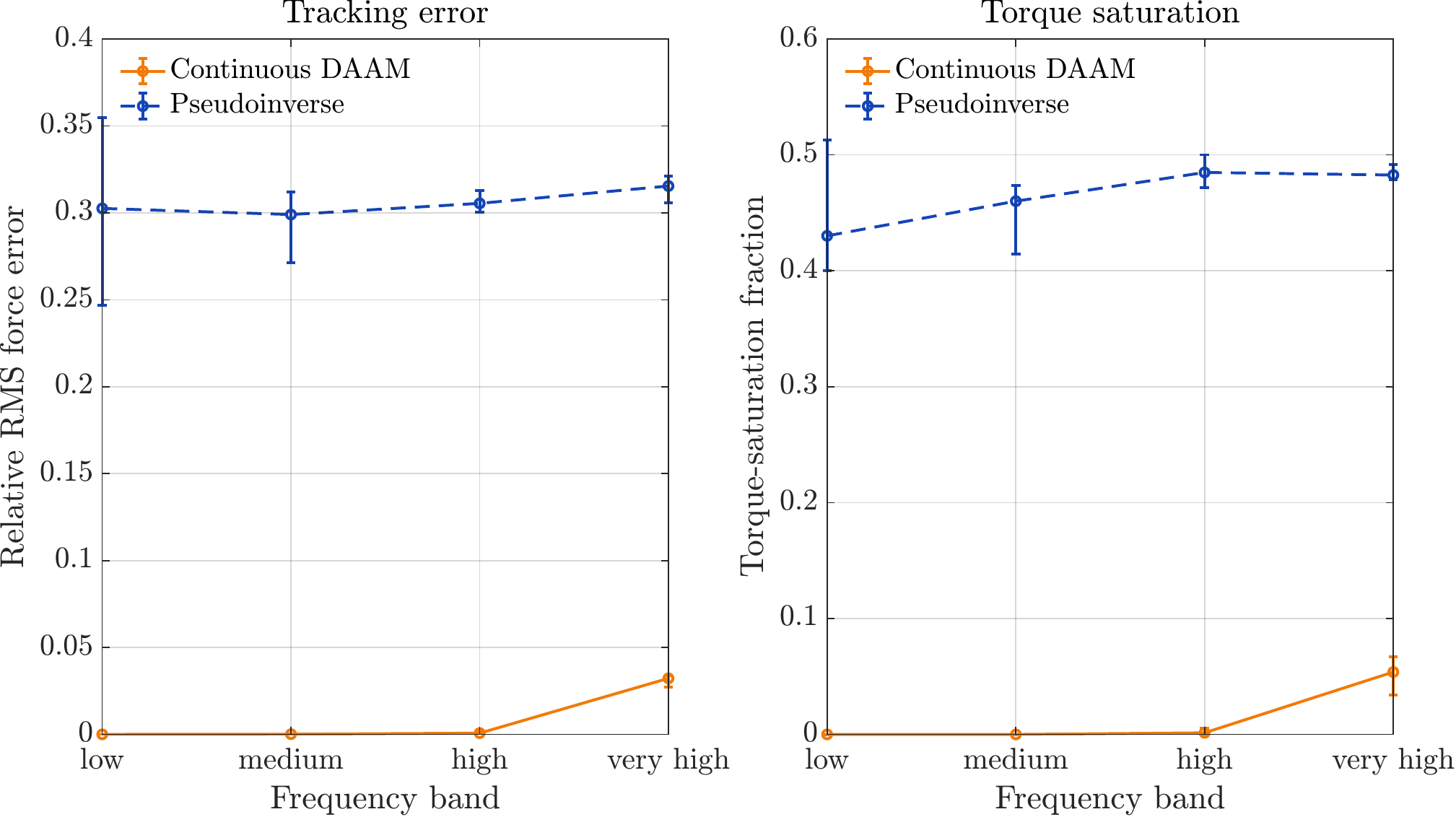}
  \subcaption{Tracking error and torque saturation by frequency band.}
  \label{fig:daam-case1-bands}
 \end{subfigure}\hfill
 \begin{subfigure}{0.49\textwidth}
  \centering
  \includegraphics[width=\linewidth]
  {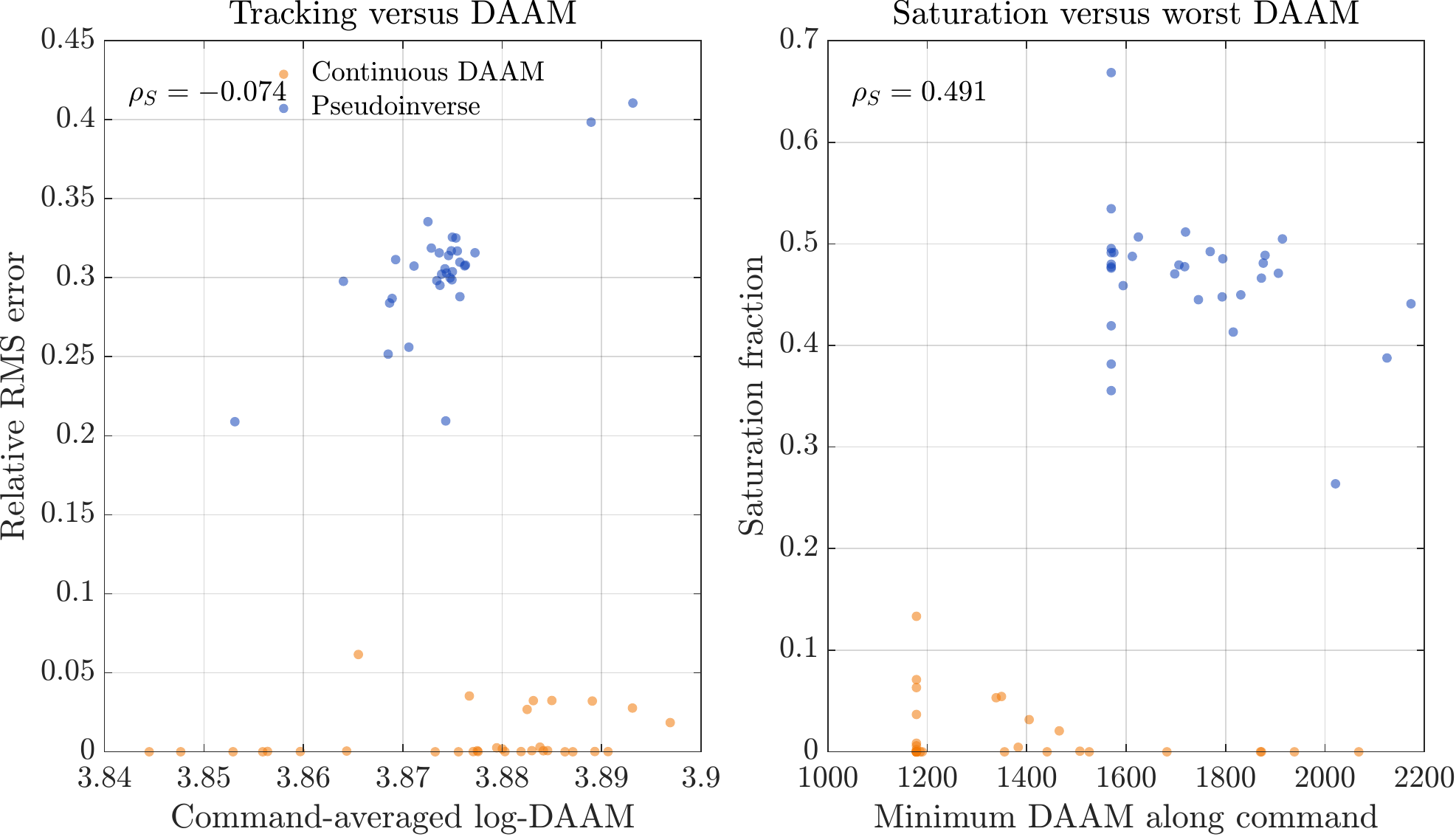}
  \subcaption{Pooled DAAM-performance relations.}
  \label{fig:daam-case1-correlation}
 \end{subfigure}
 \caption{Aggregate benchmark results for Case~I. The frequency-band
 comparison preserves the pairing between allocation sections and command
 realizations. Each point in the right panels represents one
 allocation--command pair.}
 \label{fig:daam-case1-aggregate}
\end{figure*}

The pooled Spearman coefficients in
Fig.~\ref{fig:daam-case1-correlation} are \(-0.074\) between
command-averaged log-DAAM and normalized RMS error and \(0.491\) between the
minimum DAAM encountered along a command and the saturation fraction. Because
these pooled data combine different allocations, frequency bands, and command
realizations, the coefficients are descriptive rather than direct measures of
the allocation effect. The paired comparison within each frequency band
provides the relevant performance assessment.

\subsection{Case II: Unequal Static Coefficients and Torque Limits}
\label{subsec:mixed-asymmetry-case}

The second case uses
\(
 A=(1,0.5),\;
 b=(0.1,0.1),\;
 m=(0.05,0.05),\;
 \bar\tau=(1,0.7).
\)
The physical cooperative force range is \([0,13.5]\), and the section is
constructed on \([0.405,11.47]\). Here, both the static actuation coefficients
and the torque limits are asymmetric. The pseudoinverse accounts for the
former through \eqref{eq:pseudoinverse-section}, but remains independent of
the latter.

\begin{figure*}[t]
 \centering
 \begin{subfigure}{0.43\textwidth}
  \centering
  \includegraphics[width=\linewidth]
  {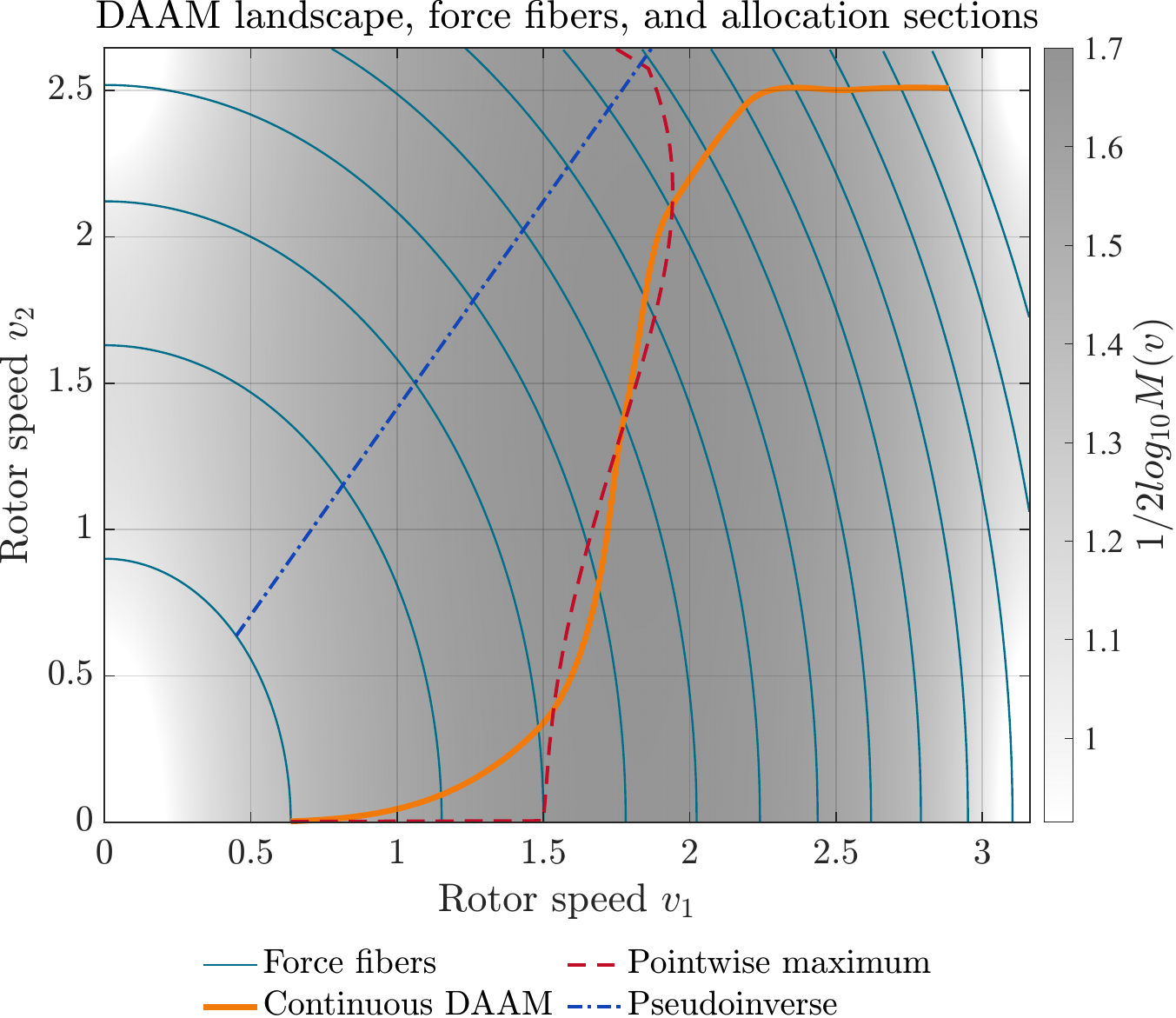}
  \subcaption{DAAM landscape, fibers, and allocation sections.}
  \label{fig:daam-case2-allocation}
 \end{subfigure}\hfill
 \begin{subfigure}{0.54\textwidth}
  \centering
  \includegraphics[width=\linewidth]
  {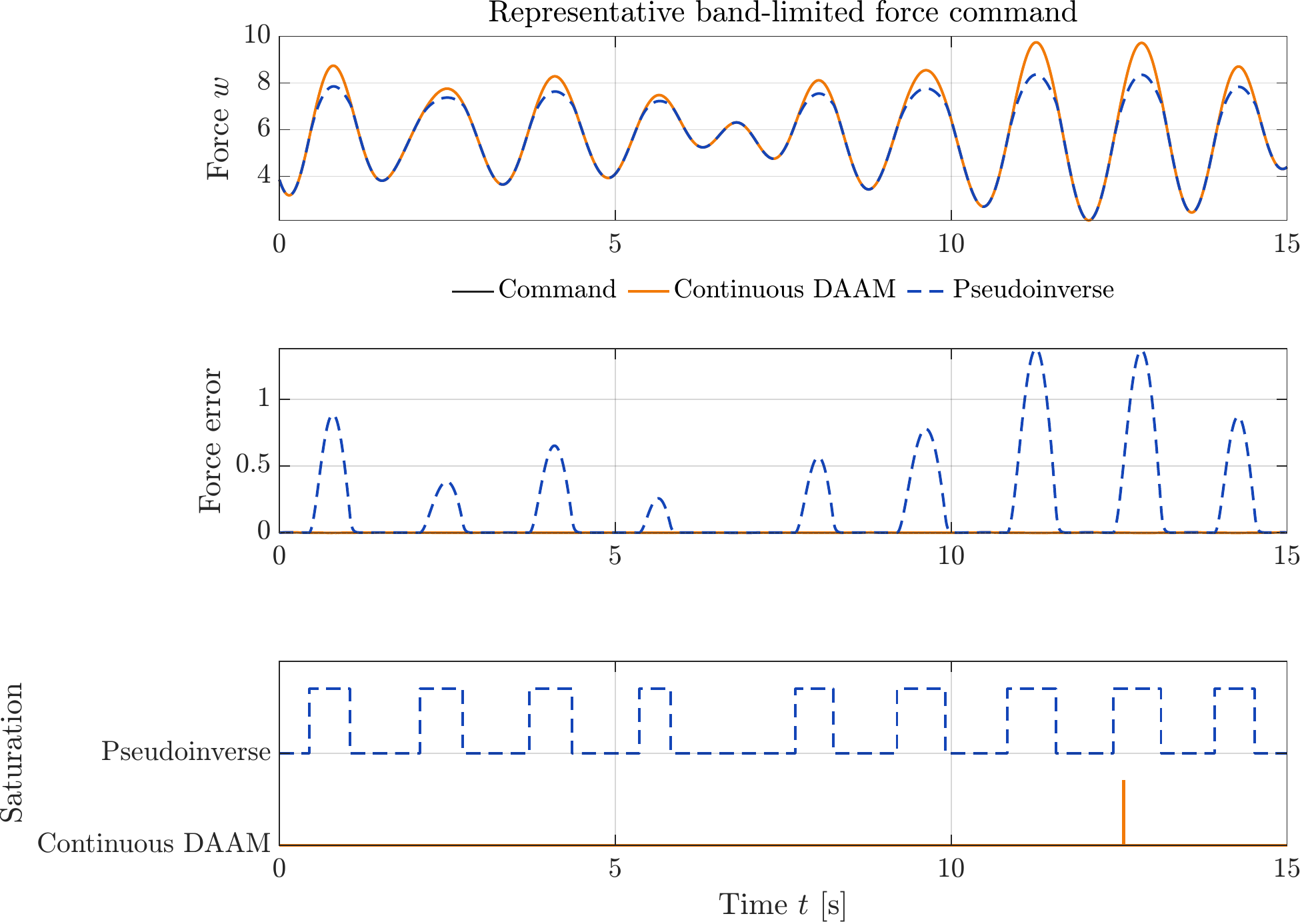}
  \subcaption{Representative force-tracking realization.}
  \label{fig:daam-case2-tracking}
 \end{subfigure}
 \caption{Allocation and saturated force tracking for Case~II. The unequal
 static coefficients alter both the allocation fibers and the pseudoinverse
 section, while the unequal torque limits affect the available acceleration
 authority.}
 \label{fig:daam-case2-main}
\end{figure*}

The continuous section in Fig.~\ref{fig:daam-case2-allocation} again follows
a favorable region of the DAAM landscape while satisfying the section-level
constraints. The pseudoinverse reflects the unequal static coefficients but
does not account for the smaller torque limit of the second rotor. The
representative response and aggregate results exhibit the same qualitative
dependence on command bandwidth as Case~I. The sections perform similarly for
slow commands, whereas saturation produces a clearer separation at higher
frequencies. In the high and very-high bands, the continuous DAAM section
exhibits lower aggregate tracking error and a smaller saturation fraction.

\begin{figure*}[t]
 \centering
 \begin{subfigure}{0.49\textwidth}
  \centering
  \includegraphics[width=\linewidth]
  {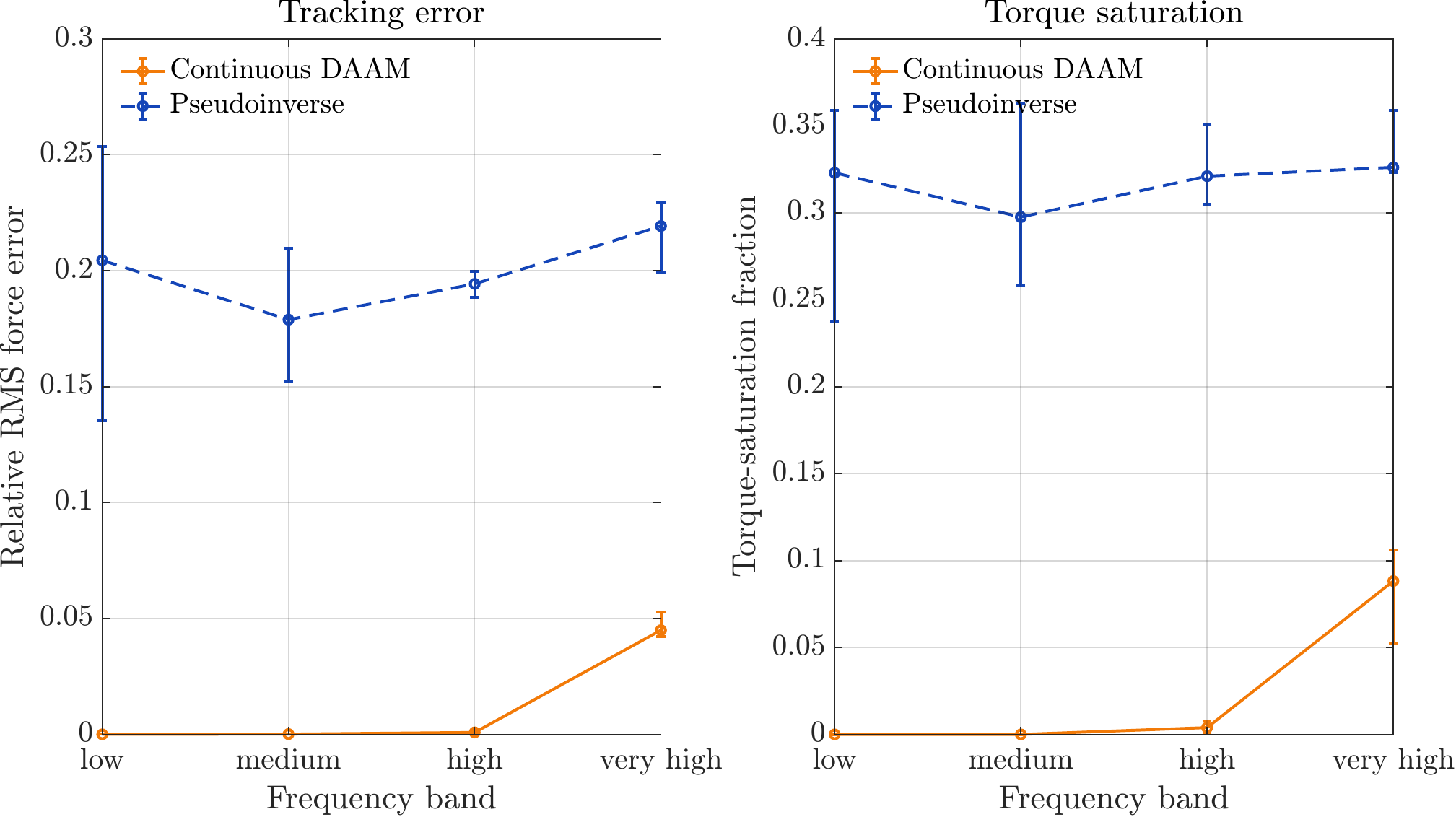}
  \subcaption{Tracking error and torque saturation by frequency band.}
  \label{fig:daam-case2-bands}
 \end{subfigure}\hfill
 \begin{subfigure}{0.49\textwidth}
  \centering
  \includegraphics[width=\linewidth]
  {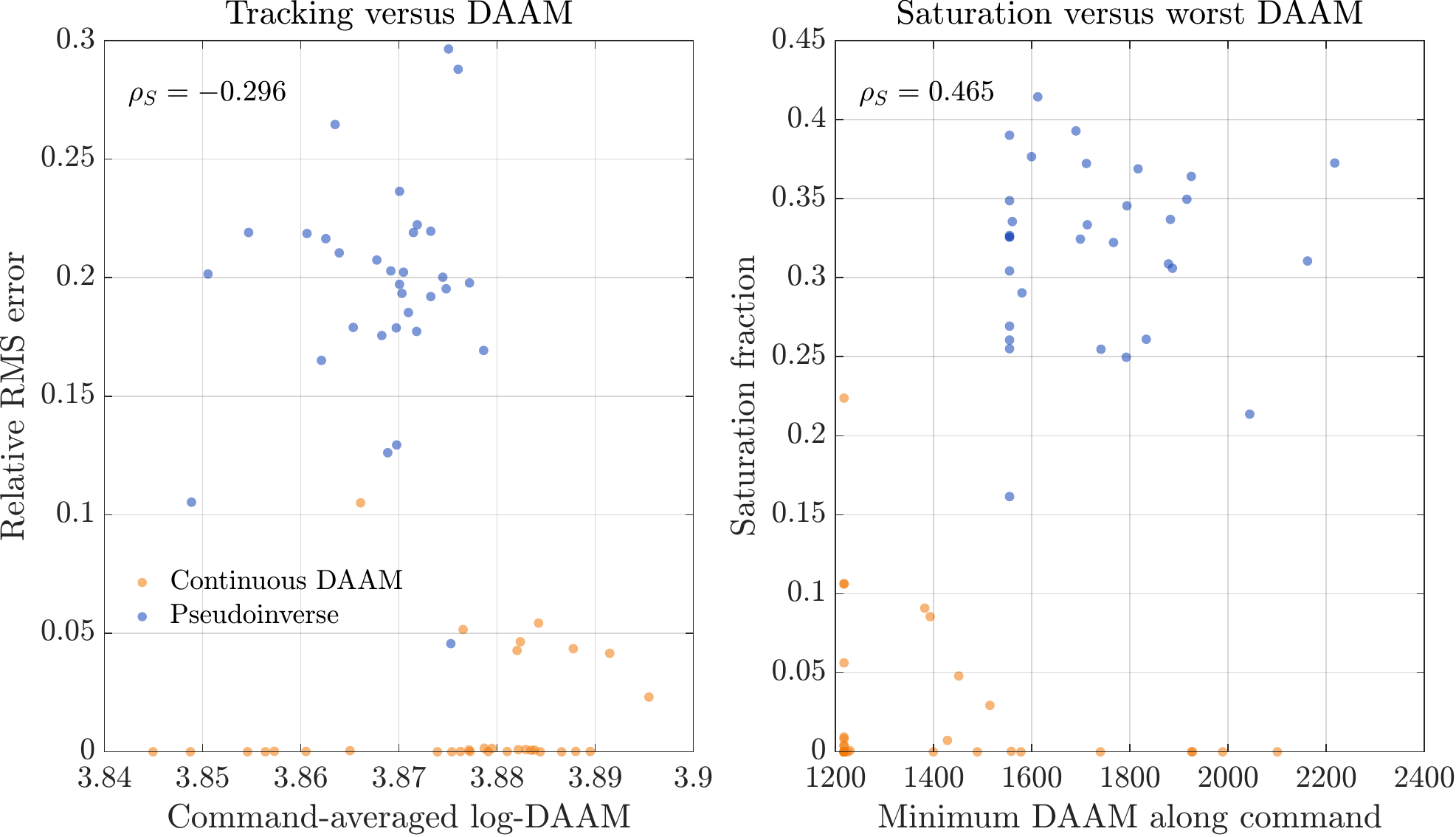}
  \subcaption{Pooled DAAM-performance relations.}
  \label{fig:daam-case2-correlation}
 \end{subfigure}
 \caption{Aggregate benchmark results for Case~II. The frequency-band
 comparison preserves the pairing between allocation sections and command
 realizations, while the right panels pool the complete benchmark data.}
 \label{fig:daam-case2-aggregate}
\end{figure*}

The pooled Spearman coefficients in
Fig.~\ref{fig:daam-case2-correlation} are \(-0.296\) between
command-averaged log-DAAM and normalized RMS error and \(0.465\) between the
minimum encountered DAAM and the saturation fraction. As in Case~I, these
coefficients combine the effects of allocation, bandwidth, and command
realization. The moderate correlations also reflect that dynamic section
following depends on \(s_i'(w)\dot w\), which contributes directly to the
required torque in \eqref{eq:section-following-torque}.

\subsection{Cross-Case Assessment}
\label{subsec:dynamic-cross-case}

The two cases isolate different forms of propulsion heterogeneity. Case~I
shows that a static pseudoinverse cannot distinguish rotors having equal
force coefficients but different drag and torque capabilities. Case~II shows
that accounting for unequal static coefficients is still insufficient when
the actuator limits differ.

The DAAM field alone does not define a dynamically feasible allocation
section. Continuity and the torque required to move along the section must
also be considered. Problem \eqref{eq:continuous-daam-section} provides one
low-dimensional construction combining these ingredients: the objective
favors local task-rate capability, while the constraints limit the torque
required over the prescribed force-rate envelope.

In both cases, the allocation-dependent performance difference emerges
principally when the command bandwidth activates the torque limits. This
behavior is consistent with DAAM representing local wrench-rate capability
rather than a steady-state allocation cost. The simulations therefore
demonstrate one use of the general SAC and DAAM framework for constructing and
assessing an executable allocation section; they do not define a general
allocator for arbitrary multirotor systems.

\section{Discussion and Limitations}
\label{sec:discussion}

The SAC is the radius of the largest zero-centered subset of the generally
asymmetric rotor-acceleration interval in
\eqref{eq:directional-interval}. The induced metric therefore represents the
acceleration authority guaranteed in both directions while omitting the
additional authority available in the more favorable direction. This
difference is particularly relevant at high rotor speeds, where increasing
\(|v_i|\) can be substantially more difficult than reducing it.

The bounded-torque rotor model is idealized. Current and voltage limits,
braking policies, battery state, thermal effects, and embedded controller
dynamics can modify the attainable acceleration set. The SAC derived in this
work is exact for \eqref{eq:rotor-dynamics}, while the subsequent geometry can
accommodate any positive centered capacity identified or certified from a
more detailed propulsion model.

The task-rate capability matrix \(M(v)\) is instantaneous and state attached,
and the DAAM index measures the volume of its capability ellipsoid. For a
scalar task, DAAM fully characterizes the radius of the bidirectional
task-rate interval. For a multidimensional task, a larger volume does not
imply greater capability in every direction; direction-specific,
minimum-radius, or task-weighted criteria may then be preferable.

Fiberwise DAAM maximization generally defines a set-valued correspondence
rather than a unique smooth allocation section. The canonical illustrations
expose multiple separated maximizing pieces and encounters with the selected
operating-domain boundary. Constructing an executable allocator therefore
also requires continuity, branch selection, and the actuator motion needed
to follow the selected states.

The illustrative construction in
Sec.~\ref{sec:continuous-allocation} addresses these requirements for a
two-rotor scalar task. The section is computed offline and depends on the
force interval, force-rate envelope, torque-use factor, and numerical
parameterization. Larger DAAM alone does not guarantee smaller tracking error,
because \(s_i'(w)\dot w\) also contributes to the required motor torque.
Nevertheless, in the two heterogeneous systems considered, the continuous
DAAM sections reduce saturation-induced force-tracking degradation relative
to the pseudoinverse in the faster command bands. These simulations provide
an allocation-level use case rather than a general allocator or evidence of
improved vehicle-level closed-loop performance.

Future work will address high-fidelity and asymmetric actuator-capability
models, multidimensional and direction-dependent objectives, online and
branch-aware allocation, comparisons with energy- and saturation-aware
methods, and vehicle-level experimental validation.

\section{Conclusion}
\label{sec:conclusion}

This work developed a capacity-aware extension of aerodynamic promptness for
heterogeneous redundantly actuated multirotors. The largest zero-centered
subset of each generally asymmetric rotor-acceleration interval defines the
SAC and induces a state-dependent actuator metric on the positive-capacity
region.

Propagating the corresponding co-metric through the nonlinear actuation
differential yields a state-attached task-rate capability matrix and
ellipsoid. Its inverse quadratic form gives the minimum normalized
rotor-acceleration effort required to generate a prescribed wrench rate,
while its volume defines the DAAM index.

Fiberwise DAAM maximization provides a task-coordinate-invariant criterion
for comparing task-equivalent actuator states. Existence, interior
optimality conditions, and the canonical illustrations show that the
resulting maximizing correspondence can be nonconvex, set valued, and
globally multipiece.

Low-dimensional examples made this geometry directly visible, while a
two-rotor scalar-task use case showed how DAAM can inform a continuous section
subject to directional motor-torque feasibility. In two heterogeneous
propulsion systems, the resulting sections reduced saturation-induced
force-tracking degradation relative to the pseudoinverse in the faster
command bands.

The general SAC, task-rate geometry, DAAM index, and fiberwise selection
theory apply to arbitrary numbers of rotors and wrench components through the
fixed actuation map and individual rotor parameters. They provide a
foundation for capability-aware allocators tailored to specific multirotor
architectures, task spaces, operating domains, and actuator constraints.

\section*{Statements and Declarations}

\subsection*{Funding}
This work was partially funded by the European Commission Horizon Europe Framework under project Autoassess (101120732).

\subsection*{Data and Code Availability}
The data generated in this work are obtained from the numerical models and
parameters reported in the manuscript.

\bibliographystyle{IEEEtran}
\bibliography{refs}

\end{document}